\documentclass[final]{cvpr}
\usepackage{times}
\usepackage{epsfig}
\usepackage{graphicx}
\usepackage{amsmath}
\usepackage{amssymb}

\usepackage{transparent}
\usepackage{xcolor}
\usepackage{booktabs}
\usepackage{multirow}

\definecolor{ForestGreen}{rgb}{0.13, 0.55, 0.13}
\usepackage{pifont}%
\newcommand{\cmark}{{\color{ForestGreen} \ding{51}}}%
\newcommand{\xmark}{{\color{red} \ding{55}}}%

\makeatletter
\@namedef{ver@everyshi.sty}{}
\makeatother
\usepackage{pgfplots}
\usepackage{pgf,tikz}

\usepackage{capt-of,etoolbox} %

\usepackage[ruled, linesnumbered]{algorithm2e}

\newcommand{\R}{\mathbb{R}}

\DeclareMathOperator*{\argmin}{arg\,min}
\DeclareMathOperator*{\argmax}{arg\,max}
\newtheorem{theorem}{Theorem}
\newtheorem{lem}[theorem]{Lemma}
\newtheorem{prop}[theorem]{Proposition}

\newenvironment{proof}{\paragraph{Proof:}}{\hfill$\square$}

\newcommand{\perm}{\mathbb{P}}
\newcommand{\orth}{\mathbb{O}}

\newcommand{\proj}{\operatorname{proj}}
\newcommand{\diag}{\operatorname{diag}}
\newcommand{\matI}{\mathbf{I}}
\newcommand{\onevec}{\mathbf{1}}

\newcommand{\shape}{\mathcal{X}}
\newcommand{\fm}{\mathcal{C}}
\newcommand{\dimLb}{b}

\renewcommand{\paragraph}{\textbf}

\def\rot{\rotatebox}

\usepackage{adjustbox}

\usepackage{enumitem}
\setitemize{noitemsep,topsep=0pt,parsep=0pt,partopsep=0pt}

\definecolor{green(ryb)}{rgb}{0.4, 0.69, 0.2}

\usepackage[pagebackref=true,breaklinks=true,colorlinks,bookmarks=false]{hyperref}

\def\cvprPaperID{} 
\def\confYear{}

\begin{document}

\title{Isometric Multi-Shape Matching}

\author{ {Maolin Gao$^\dagger$ ~~~~~ Zorah L{\"a}hner$^\dagger$ ~~~~~ Johan Thunberg$^\ddagger$ ~~~~~ Daniel Cremers$^\dagger$ ~~~~~ Florian Bernard$^\dagger$}\\
$^\dagger$ Technical University of Munich\\
$^\ddagger$ Halmstad University
}

\makeatletter
\let\@oldmaketitle\@maketitle%
\renewcommand{\@maketitle}{\@oldmaketitle%
  \centerline{%
  \begin{minipage}{\linewidth}
  \centering
  \begin{tabular}{c|c|c}
        \adjustbox{valign=t}{\definecolor{bostonuniversityred}{rgb}{0.8, 0.0, 0.0}
\begin{tikzpicture}[scale=0.87, transform shape]%
\node[anchor=south west,inner sep=0] (image) at (0,0) {\includegraphics[width=0.4\textwidth]{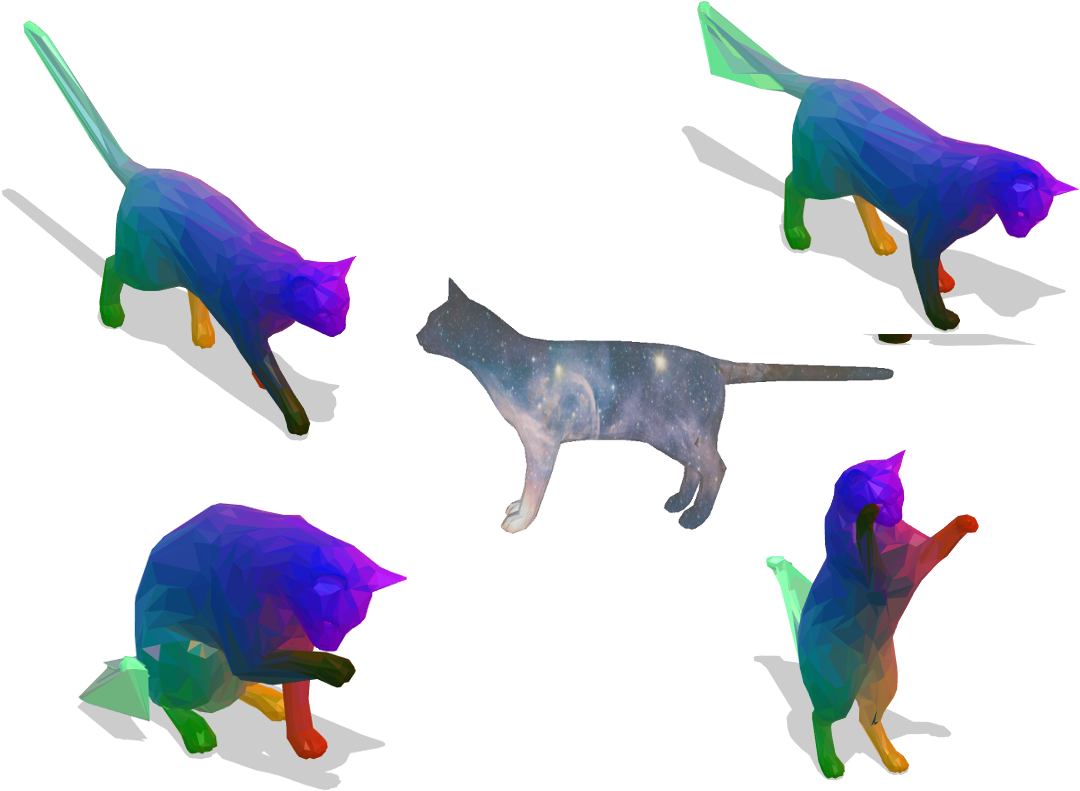}};
    \begin{scope}[x={(image.south east)},y={(image.north west)}]
    \end{scope}
    \node [anchor=west] (U1) at (2.8,3.3) {};
    \node [anchor=west] (U2) at (3.8,3) {};
    \node [anchor=west] (U3) at (2.8,2) {};
    \node [anchor=west] (U4) at (4.3,1.7) {};
    \node [anchor=west] (C1) at (0.4,4.7) {};
    \node [anchor=west] (C2) at (5,4) {};
    \node [anchor=west] (C3) at (2.2,0.8) {};
    \node [anchor=west] (C4) at (4.9,0.2) {};
    
    \node (C) at (0.5,4.9) {};
    \node (B) at (3.5,3.3) {};
    \node (A) at (5.1,4.4) {};
    \draw [very thick,->] (C1) edge[bend left=40] node [below left] {$P_1 $} (U1);
    \draw [very thick,->] (C2) edge[bend right=30] node [below right] {$P_2$} (U2);
    \draw [very thick,->] (C3) edge[bend right=30] node [right] {$P_3$} (U3);
    \draw [very thick,->] (C4) edge[bend left=30] node [left] { $P_4$} (U4);
    \draw [very thick,->,bostonuniversityred] (C) to[out=10, in=120] node [above right,text width=1.5cm,pos=0.6] {\large $P_1P_2^\top$ } (B.center) to[out=60, in=-170] (A);
\end{tikzpicture}%
}  &
        \adjustbox{valign=t}{%
        \begin{tabular}{ccc}%
          \includegraphics[width=.088\linewidth]{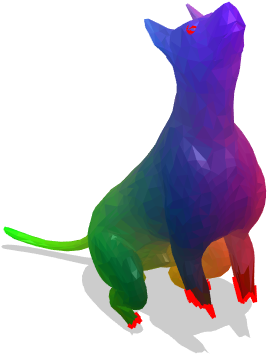} & \includegraphics[width=.088\linewidth]{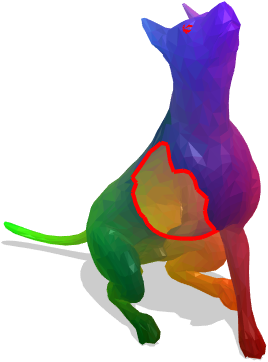}\\%
          \includegraphics[width=.088\linewidth]{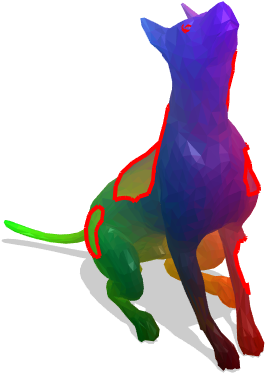} & \includegraphics[width=.088\linewidth]{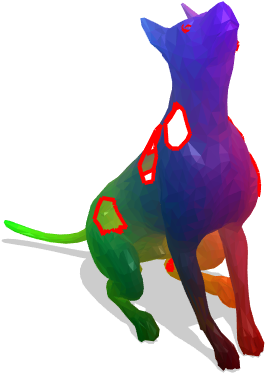}  
        \end{tabular}%
        } ~& %
        \adjustbox{valign=t}{
        \begin{tabular}{cc}%
          \ & \ \\
          \includegraphics[width=.19\linewidth]{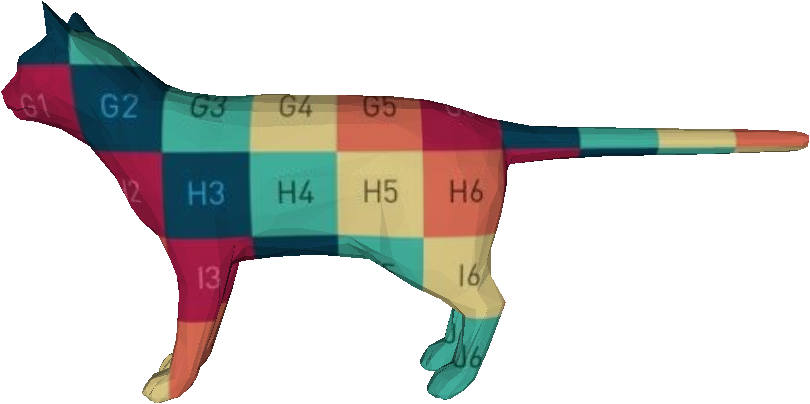} &
          \includegraphics[width=.13\linewidth]{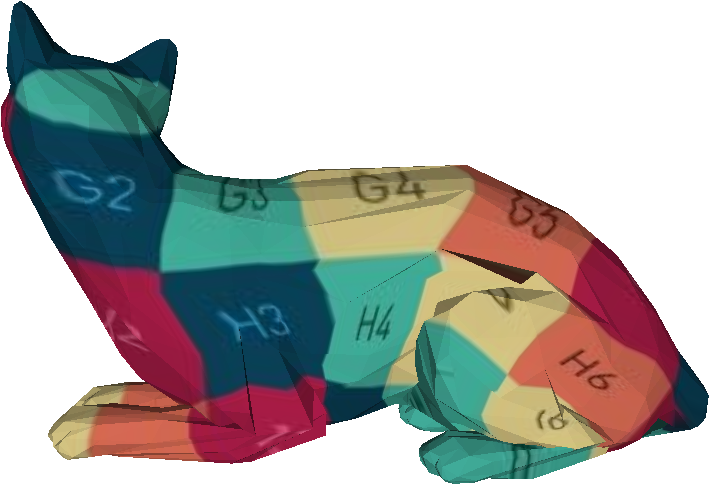}\\%
          \includegraphics[width=.19\linewidth]{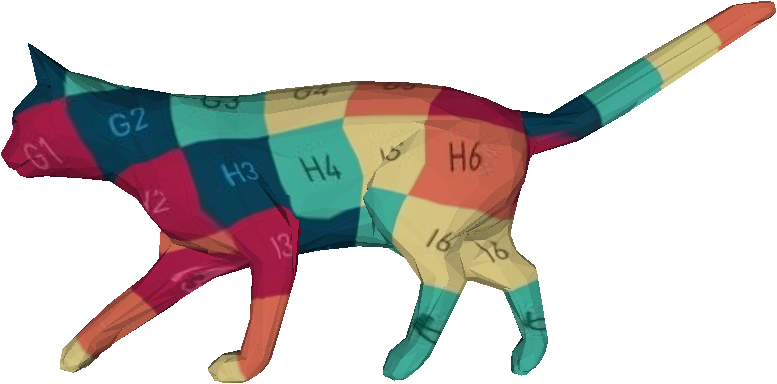} &
          \includegraphics[width=.13\linewidth]{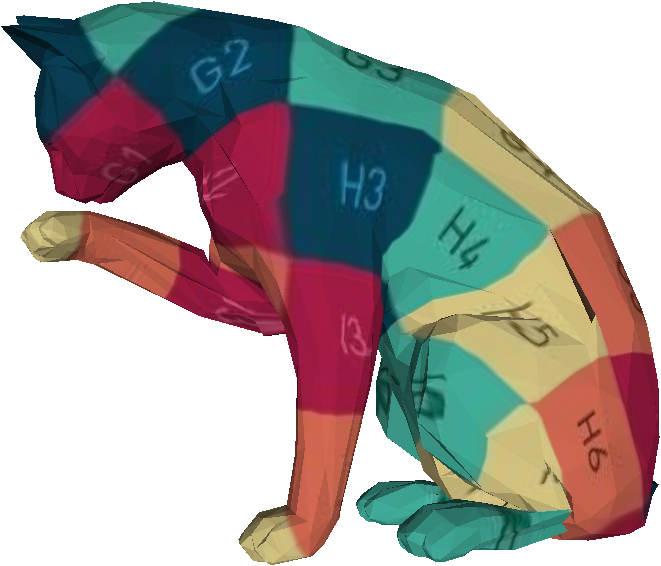}
        \end{tabular}%
        }
    \end{tabular}
  \end{minipage}
  }%
    \captionof{figure}{\textbf{Left:} We present a novel approach for isometric multi-shape matching based on matching each shape to a (virtual) universe shape (shown semi-transparent). Our formulation represents  point-to-point correspondences between shapes $i$ and $j$ as the composition of the shape-to-universe permutation matrix $P_i$ and the universe-to-shape permutation matrix $P_j^T$. By doing so, the pairwise matchings $P_{ij} = P_iP_j^\top$ are by construction cycle-consistent. \textbf{Middle:} Our formulation successfully solves isometric multi-matching of partial shapes.  \textbf{Right:} Due to the cycle-consistency we can use our correspondences to faithfully transfer textures across a shape collection.
    }
    \label{fig:teaser}
  {}\bigskip}%
\makeatother

\newcommand\myfigure{%
  \centerline{%
  \begin{minipage}{\linewidth}
  \centering
  \begin{tabular}{c|c|c}
        \adjustbox{valign=t}{\definecolor{bostonuniversityred}{rgb}{0.8, 0.0, 0.0}
\begin{tikzpicture}[scale=0.87, transform shape]%
\node[anchor=south west,inner sep=0] (image) at (0,0) {\includegraphics[width=0.4\textwidth]{figures/universe00.png}};
    \begin{scope}[x={(image.south east)},y={(image.north west)}]
    \end{scope}
    \node [anchor=west] (U1) at (2.8,3.3) {};
    \node [anchor=west] (U2) at (3.8,3) {};
    \node [anchor=west] (U3) at (2.8,2) {};
    \node [anchor=west] (U4) at (4.3,1.7) {};
    \node [anchor=west] (C1) at (0.4,4.7) {};
    \node [anchor=west] (C2) at (5,4) {};
    \node [anchor=west] (C3) at (2.2,0.8) {};
    \node [anchor=west] (C4) at (4.9,0.2) {};
    
    \node (C) at (0.5,4.9) {};
    \node (B) at (3.5,3.3) {};
    \node (A) at (5.1,4.4) {};
    \draw [very thick,->] (C1) edge[bend left=40] node [below left] {$P_1 $} (U1);
    \draw [very thick,->] (C2) edge[bend right=30] node [below right] {$P_2$} (U2);
    \draw [very thick,->] (C3) edge[bend right=30] node [right] {$P_3$} (U3);
    \draw [very thick,->] (C4) edge[bend left=30] node [left] { $P_4$} (U4);
    \draw [very thick,->,bostonuniversityred] (C) to[out=10, in=120] node [above right,text width=1.5cm,pos=0.6] {\large $P_1P_2^\top$ } (B.center) to[out=60, in=-170] (A);
\end{tikzpicture}%
}  &
        \adjustbox{valign=t}{%
        \begin{tabular}{ccc}%
          \includegraphics[width=.088\linewidth]{figures/partial/part10_target.png} & \includegraphics[width=.088\linewidth]{figures/partial/part2_target.png}\\%
          \includegraphics[width=.088\linewidth]{figures/partial/part3_target.png} & \includegraphics[width=.088\linewidth]{figures/partial/part7_target.png}  
        \end{tabular}%
        } ~& %
        \adjustbox{valign=t}{
        \begin{tabular}{cc}%
          \ & \ \\
          \includegraphics[width=.19\linewidth]{figures/qual/cat00.png} &
          \includegraphics[width=.13\linewidth]{figures/qual/cat03.png}\\%
          \includegraphics[width=.19\linewidth]{figures/qual/cat02.png} &
          \includegraphics[width=.13\linewidth]{figures/qual/cat01.png}
        \end{tabular}%
        }
    \end{tabular}
  \end{minipage}
  }%
    \captionof{figure}{\textbf{Left:} We present a novel approach for isometric multi-shape matching based on matching each shape to a (virtual) universe shape (shown semi-transparent). Our formulation represents  point-to-point correspondences between shapes $i$ and $j$ as the composition of the shape-to-universe permutation matrix $P_i$ and the universe-to-shape permutation matrix $P_j^T$. By doing so, the pairwise matchings $P_{ij} = P_iP_j^\top$ are by construction cycle-consistent. \textbf{Middle:} Our formulation successfully solves isometric multi-matching of partial shapes.  \textbf{Right:} Due to the cycle-consistency we can use our correspondences to faithfully transfer textures across a shape collection.
    }
    \label{fig:teaser}
  } 
  
\maketitle

\begin{abstract}
Finding correspondences between shapes is a fundamental problem in computer vision and graphics, which is relevant for many applications, including 3D reconstruction, object tracking, and style transfer. The vast majority of correspondence methods aim to find a solution between pairs of shapes, even if multiple instances of the same class are available. While isometries are often studied in shape correspondence problems, they have not been considered explicitly in the multi-matching setting. This paper closes this gap by proposing a novel optimisation formulation for isometric multi-shape matching. We present a suitable optimisation algorithm for solving our formulation and provide a convergence and complexity analysis. Our algorithm obtains multi-matchings that are by construction provably cycle-consistent. We demonstrate the superior performance of our method on various datasets and set the new state-of-the-art in isometric multi-shape matching.
\end{abstract}

\section{Introduction}
The identification of correspondences between 3D shapes, also known as the shape matching problem, is a longstanding challenge in visual computing. Correspondence problems have a high relevance due to their plethora of applications, including 3D reconstruction, deformable object tracking, style transfer, shape analysis, or general data canonicalisation,~e.g.~to facilitate  learning by establishing a common vector space representation. 

There are certain problem formulations that cover  generic correspondence problems involving different types of data and varying application scenarios. One example is the widely-studied quadratic assignment problem (QAP)~\cite{Lawler:1963wn}.
 Due to its NP-hardness~\cite{Pardalos:1993uo}, reasonably large QAPs cannot be solved satisfactorily in most practical settings. However, contrary to bringing generic objects (e.g.~graphs) into correspondence, when considering 3D shapes it is often possible to exploit particular structural properties in order to effectively solve  the shape matching problem.

For example, it has been demonstrated that explicitly modelling the low-dimensional structure of shape matching problems often allows to find global optima for a wide range of shape matching formulations~\cite{bernard2020mina}. It was also shown that learning suitable feature representations from shapes  improves the matching performance drastically compared to using hand-crafted features~\cite{litany2017deep}.

Moreover, when assuming \emph{(near)-isometries} between shapes,  efficient and powerful spectral approaches can be leveraged for shape matching~\cite{ovsjanikov2012functional}.
Isometries describe classes of deformable shapes of the same type but in different poses, \eg humans or animals who are able to adopt a variety of poses. Potential applications for isometric shape matching include AR/VR or template matching.
While (near)-isometric shape matching has been studied extensively for the case of matching a pair of shapes, the isometric multi-shape matching problem, where an entire collection of (near-isometric) shapes is to be matched, is less explored. Important applications of isometric multi-shape matching include learning low-dimensional shape space representations~\cite{zuffi20173d}, motion tracking and reconstruction.

In principle, any pairwise shape matching method can be used for matching a shape collection. To do so, one can select one of the shapes as reference, and then solve a sequence of pairwise shape matching problems between each of the remaining shapes and the reference. However, a major disadvantage is that such an approach has a strong bias due to the choice of the reference. 
Alternatively, one could solve pairwise shape matching problems between all pairs of shapes in the shape collection. Although this way there is no bias, in general the resulting correspondences are not cycle-consistent. As such, matching shape A via shape B to shape C, may lead to a different correspondence than matching shape A directly to C.

In order to achieve cycle consistency, so-called \emph{permutation synchronisation} methods can be used as post-processing~\cite{Pachauri:2013wx}.
A disadvantage of synchronisation-based multi-shape matching is that it is a two-stage procedure, where  pairwise matchings are obtained in the first proceedure, and synchronization is assured in the second. With that, the matching results are often suboptimal -- even if one reverts to an alternating procedure using a soft coupling \cite{STCB-07a}. 
For \emph{isometric} multi-shape matching this sequential procedure is particularly disadvantageous, since during the second stage the very strong prior about the isometric nature of the shapes is completely ignored. 

Although  multi-matchings obtained by synchronisation procedures are cycle-consistent, the matchings are often spatially non-smooth and noisy, as we illustrate in Sec.~\ref{sec:exp}. 
From a theoretical point of view, the most appropriate approach for addressing multi-shape matching is based on a unified formulation, where  cycle consistency is assured already when the multi-matchings are computed. Although some approaches  fit into this category~\cite{cosmo2017consistent,bernard2019hippi}, none of the existing methods are tailored explicitly towards isometric multi-shape matching in order to take full advantage in this setting.

In this work we fill this gap by introducing a generalisation of state-of-the-art isometric two-shape matching approaches towards isometric multi-shape matching. We demonstrate that explicitly exploiting the isometry property leads to a natural and elegant formulation that achieves improved results compared to previous methods.
Our main contributions can be summarised as:
\begin{itemize}
    \item A \textbf{novel optimisation formulation} for isometric multi-shape matching.
    \item An efficient and easy-to-implement algorithm with \textbf{provable convergence}.
    \item \textbf{Guaranteed cycle-consistency} without enforcing explicit constraints.
    \item \textbf{Improvements over the state-of-the-art} on various shape matching benchmarks.
\end{itemize}
\section{Related Work}\label{sec:related}
\paragraph{Assignment problems.}
Shape matching can be formulated as bringing points defined on one shape into correspondence with points on another shape. A simple mathematical formulation for doing so is the linear assignment problem (LAP)~\cite{Munkres:1957ju}, where a linear cost function is optimised over the set of permutation matrices. The objective function defines the cost for matching points on the first shape to points on the second shape. In shape matching, the costs are typically computed based on feature descriptors, such as the heat kernel signature~\cite{bronstein2010scale}, wave kernel signature~\cite{aubry2011wave}, or SHOT~\cite{salti2014shot}.
Despite the exponential size of the search space, there exist efficient polynomial-time algorithms to solve the LAP~\cite{Bertsekas:1998vt}. A downside of the LAP is that the geometric relation between points is not explicitly taken into account, so that the found matchings lack spatial smoothness. To compensate for this, a correspondence problem formulation based on the quadratic assignment problem (QAP)~\cite{Koopmans:1957gf,Lawler:1963wn,Pardalos:1993uo,Burkard:2009hp,Loiola:2ua4FrR7} can be used. In that case, in addition to linear point-to-point matching costs, quadratic costs for matching \emph{pairs of points} on the first shape to pairs of points on the second shape are taken into account. Since pairs of points can be understood as edges in a graph, this corresponds to graph matching. Due to the NP-hardness of the QAP~\cite{Pardalos:1993uo}, there are no  algorithms that can reliably find global optima efficiently for large (non-trivial) problem instances. In addition to exhaustive search algorithms that have exponential worst-case time complexity~\cite{Bazaraa:1979fh}, there are various more efficient but non-optimal solution strategies. They include spectral methods~\cite{Leordeanu:2005ur,Cour:2006un}, convex relaxations~\cite{Zhao:1998wc,Fogel:2013wt,Torresani:2013gj,kezurer2015,swoboda2017study}, some of them relying on path-following~\cite{Zaslavskiy:2009fq,Zhou:2016ty,Dym:2017ue,bernard:2018}, as well as various non-convex formulations~\cite{le2017alternating,solomon2016entropic,Vestner:2017tj,holzschuh20simanneal}.
For suitably defined matching costs the QAP is an appropriate formalism for \emph{modelling} isometric shape matching. However, due to its NP-hardness the QAP is computationally very difficult to solve. Moreover, due to the generality of the formalism, it does not fully exploit the structural properties present in isometric shape matching problems, and is therefore a suboptimal choice from a computational perspective.

\paragraph{Isometric shape matching.} 
The near-isometric shape correspondence problem has been studied extensively in the literature, see \cite{sahillioglu19survey} for a recent survey. 
Apart from methods tackling a QAP formulation (see previous paragraph), there exist directions utilising other structural properties of isometries. 
The Laplace-Beltrami operator (LBO)~\cite{pinkall93cotan}, a generalisation of the Laplace operator on manifolds, as well as its eigenfunctions are invariant under isometries.
Methods like \cite{mateus08laplacian, melzi19zoomout} directly incorporate this knowledge into the pipeline, or use descriptors based on these \cite{aubry2011wave, vestner2017efficient, bronstein2010scale}.
Functional maps \cite{ovsjanikov2012functional} reformulate the point-wise correspondence problem as a correspondence between functions. 
The functional mapping is represented as a low-dimensional matrix for suitably chosen basis functions. The classic choice are the eigenfunctions of the LBO, which are invariant under isometries and predestined for this setting. Moreover, for general non-rigid settings learning these basis functions has also been proposed  \cite{marin20embedding}. 
A wide variety of extensions to make functional maps more robust or more flexible have been developed. This includes orientation-preservation \cite{ren2018orientation}, image co-segmentation \cite{wang2013image}, denoising \cite{ezuz2017deblur,ren19struct}, partiality \cite{rodola2016partial}, and non-isometries \cite{eisenberger20smoothshells}.
However, extracting a point-wise correspondence from a functional map matrix is not trivial \cite{corman15continuous,rodola2017regularized}. This is mainly because of the low-dimensionality of the functional map, and the fact that not every functional map matrix is a representation of a point-wise correspondence \cite{ovsjanikov2012functional}. 
In~\cite{maron2016point}, the authors simultaneously solve for point-wise correspondences and functional maps for non-rigid shape matching.

Due to their low-dimensionality and continuous representation, functional maps also serve as the backbone of many deep learning architectures for 3D correspondence.
One of the first examples is FMNet \cite{litany2017deep}, which has also been extended for unsupervised learning settings recently \cite{halimi2019unsupervised, aygun2020heatkernel, roufosse19unsupervised}.
Other learning methods rely on a given template for each class \cite{groueix183dcoded} or local neighbourhood encoding to learn a compact representation \cite{lim18intrinsic}.
The recently conducted SHREC correspondence contest on isometric and non-isometric 3D shapes \cite{dyke19shrec} revealed that there is still room for improvement in both fields.

\paragraph{Generic multi-matching.}
The multi-matching problem is relatively well-studied for generic settings,~e.g.~for matching multiple graphs~\cite{yan2015consistency,yan2015multi,shi2016tensor,bernard:2018,swoboda2019convex,wang2020graduated}, or matching keypoints in image collections~\cite{Wang:2017ub,Tron:kUBrCZhd,ma2020image}. A desirable property of multi-matchings is cycle consistency (which we will formally define in Sec.~\ref{sec:matchingRep}).
Establishing cycle consistency in a given set of pairwise
matchings, known as permutation synchronisation,
has been addressed extensively in the literature~\cite{Nguyen:2011eb,Pachauri:2013wx,Huang:2013uk, Chen:2014uo,zhou2015multi,Shen:2016wx,Tron:kUBrCZhd,Maset:YO8y6VRb,Schiavinato:2017fr,bernard2019synchronisation}.

\begin{figure}[h!]
    \centering
    \includegraphics[width=.26\linewidth]{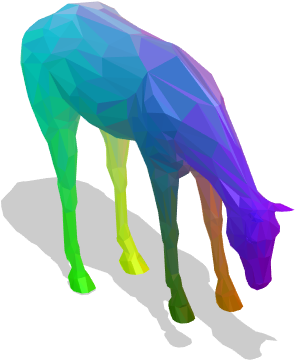} \ \ 
    \includegraphics[width=.26\linewidth]{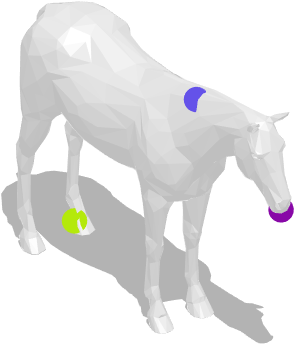} \ \ 
    \includegraphics[width=.26\linewidth]{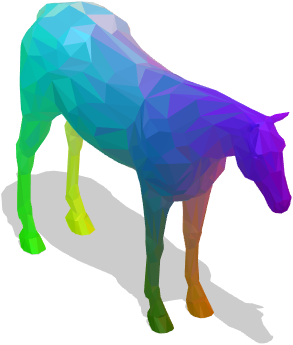}\\
    \footnotesize{Colour legend ~~~~~~~~~~~~~~~~ Cosmo et al.~~~~~~~~~~~~~~~~~~~~Ours}
    \caption{The method by Cosmo et al.~\cite{cosmo2017consistent} leads to extremely sparse multi-matchings (middle), whereas our method obtains dense matchings (right).}
    \label{fig:oursCosmo} 
\end{figure}

\paragraph{Multi-shape matching.} There are various works that particularly target the matching of multiple shapes.  In~\cite{Huang:2013uk,kezurer15tight}, semidefinite programming relaxations  are proposed for the multi-shape matching problem. However,  due to the employed lifting strategy, which drastically increases the number of variables, these methods are not scalable to large problems and only sparse correspondences are obtained. In~\cite{cosmo2017consistent}, a game-theoretic formulation for establishing  multi-matchings is introduced. Due to the use of a sparse modelling approach, the method also has the disadvantage that only few points per shape are matched, see Fig.~\ref{fig:oursCosmo}.
In~\cite{huang2019tensor}, tensor maps are introduced for synchronising heterogeneous shape collections using a low-rank tensor decomposition formulation.  The work~\cite{groueix19cycleconsistent} presents a self-supervised learning approach for finding surface deformations.  A higher-order projected power iteration approach was presented in~\cite{bernard2019hippi}, which was applied to various multi-matching settings, such as multi-image matching or multi-shape matching. 
A shortcoming when applying the mentioned multi-shape matching approaches to  \emph{isometric} settings is that they do not exploit structural properties of isometric shapes. Hence, they lead to suboptimal multi-matchings, which we experimentally confirm in Sec.~\ref{sec:exp}. One exception is the recent work on spectral map synchronisation~\cite{huang2020consistent}, which builds upon functional maps and is, in principal, well-suited for isometric multi-shape matching. However, although the authors take into account cycle consistency, respective penalties are only imposed on pairwise functional maps, rather than on the point-wise correspondences. In Sec.~\ref{sec:exp} we demonstrate that it leads to multi-matchings that have large cycle errors.

\section{Background}
In this section we introduce our representation for multi-matchings of 3D shapes, formalise the notion of cycle consistency, and provide a recap of functional maps.

\subsection{Multi-Matching Representation}\label{sec:matchingRep}
We are given a collection $\shape_1, \ldots, \shape_k$ of $k$ 3D shapes, where each shape is a triangular surface mesh that discretises a two-dimansional Riemannian manifold. The $i$-th shape $\shape_i$ is represented by a total of $m_i$  vertices in 3D space.
For any two non-negative integers $s$ and $t$, 
\begin{align}\label{eq:perm}
    \perm_{st} = \{ P \in \{0,1\}^{s \times t} : P \onevec_t \leq \onevec_s, \onevec_s^\top P \leq \onevec_t^\top  \}\,,
\end{align}
is the set of partial permutation matrices, 
where $\onevec_s$ is the $s$-dimensional column vector with each element \mbox{equals to~$1$.}
As such, correspondences between vertices of pairs of shapes $\shape_i$ and $\shape_j$ can be represented by using the partial permutation matrix $P_{ij} \in \perm_{m_i m_j}$. To be more specific, if the element  at position $(u,v)$ in $P_{ij}$ has the value $1$, the \mbox{$u$-th} vertex of $\shape_i$ is said to be in correspondence with the $v$-th vertex of $\shape_j$. We assume $P_{ii} = \matI_{m_i}$, where $\matI_{m_i}$ denotes the  identity matrix of size $m_i$, and that all pairwise matchings are symmetric in the sense that $P_{ij} = P_{ji}^\top$. 

\paragraph{Cycle consistency (pairwise).} For bijective matchings, in which case the $P_{ij}$ are full permutation matrices (the inequalities in~\eqref{eq:perm} become equalities), cycle consistency means that for all $i,j,\ell \in \{1,\ldots,k\}$, it holds that
\begin{align}\label{eq:cyccons}
    P_{ij}P_{j\ell} = P_{i\ell} \,.
\end{align}
Cycle consistency is a natural property and constitutes a necessary condition for the pairwise matchings to correspond to the ground truth. As such, cycle consistency can serve as additional constraint in order to better restrict the space of solutions in multi-matching problems. 

\paragraph{Cycle consistency (universe).} Instead of using the explicit cycle consistency constraints in~\eqref{eq:cyccons}, 
one can represent multi-matchings by using \emph{shape-to-universe} matchings~\cite{Pachauri:2013wx,Tron:kUBrCZhd,bernard2019synchronisation}. 
In this case, cycle consistency holds implicitly without having to enforce the constraints \eqref{eq:cyccons} in the problem formulation, and without having to develop a customised solution strategy.
The union of all \emph{distinct} points across all $k$ shapes are called \emph{universe points}, and we use $d$ to denote the total number of universe points. The shape-to-universe formulation of cycle consistency also applies to the case of partial multi-matchings, which is the setting we are interested in.
The main idea of the shape-to-universe representation is that each point in each of the $k$ shapes is brought into correspondence with exactly one of the universe points. Then, all points across the $k$ shapes that are in correspondence with the same universe point are said to be in correspondence with each other. Mathematically, let $P_i \in \perm_{m_id}$ be the partial permutation matrix that represents the matching of the $i$-th shape to the universe. Since each of the $m_i$ points is assigned to exactly one universe point, we have $P_i \onevec_d = \onevec_{m_i}$. Pairwise matchings can be obtained from the shape-to-universe matchings via
\begin{align}\label{eq:cycconsuniverse}
    P_{ij} = P_i P_j^\top \,.
\end{align}
The intuition is that the matching from $i$ to $j$ can be represented as matching $i$ to the universe, followed by matching the universe to $j$, which is illustrated in Fig.~\ref{fig:teaser}.

For our later elaborations it will be convenient to stack all $P_i$'s into a tall block-matrix, which we define as 
\begin{align}
    U = \begin{bmatrix} P_1^\top, P_2^\top, \ldots P_k^\top \end{bmatrix}^\top\,.
    \label{eq:stackedU}
\end{align}
\mbox{The matrix $U$ is $(m {\times} d)$-dimensional, where $m = \sum_{i=1}^k m_i$.} Moreover, we introduce the blockwise partial permutation constraint notation $U \in \perm$ (without subscript in $\perm$) to indicate that for each block $P_i$ in $U$ it holds that $P_i \in \perm_{m_id}$ and $P_i \onevec_d = \onevec_{m_i}$. We emphasise that by representing multi-matchings in terms of the matrix $U$, the resulting pairwise matchings are, by definition, cycle-consistent.

\subsection{Functional Maps} \label{sec:fms}

Functional Maps \cite{ovsjanikov2012functional} formulate the correspondence problem as a linear mapping $\fm_{ij}: L^2(\shape_i) \to L^2(\shape_j)$ between function spaces on the surfaces of $\shape_i, \shape_j$, rather than as a point-to-point correspondence between vertices.
Let $\Phi_i \in \mathbb{R}^{m_i \times \dimLb}, \Phi_j \in \mathbb{R}^{m_j \times \dimLb}$ be the first $\dimLb$ eigenfunctions of the Laplace-Beltrami operator (LBO) \cite{pinkall93cotan}. Then $\fm_{ij}$ transfers the function $F$ represented in the basis $\Phi_i$ to the function $G$ represented in the basis $\Phi_j$,~i.e.
\begin{align}
    \fm_{ij}(\Phi_i^\dagger F) = \Phi_j^\dagger G\,.
\end{align}
Here, $\Phi_{\bullet}^\dagger$ denotes the Moore-Penrose pseudoinverse of $\Phi_{\bullet}$.
In particular, the optimal $\fm_{ij}$ will map compatible
functions $F \in L^2(\shape_i)$ and $G \in L^2(\shape_j)$, \eg descriptor functions or indicator functions on corresponding points, onto each other.
We will use $\fm$ without subscripts to describe common properties of all $\fm_{ij}$.
Due to the linearity of $\fm$, it can be written as a matrix.
Orthogonality of $\fm$ is related to area preservation in the correspondence \cite{ovsjanikov2012functional} which is also a property of isometries.
Thus, we use orthogonality as a prior by projecting all $\fm$'s onto the set of orthogonal matrices
\begin{align}
    \orth_b = \{ \fm \in \mathbb{R}^{\dimLb \times \dimLb}: \fm \fm^\top = \textbf{I}_b \}\,.
    \label{eq:fmspace}
\end{align}
Similar to the previous section, we want to impose cycle consistency on the pairwise functional maps $\fm_{ij}$. 
We do so by defining a shape-to-universe functional map $\fm_i$ from $\shape_i$ to a (virtual) universe shape. We achieve cycle consistency by composing each pairwise functional map using shape-to-universe functional maps,~i.e.
\begin{align}
    \fm_{ij} = \fm_i \fm_j^\top \,.
    \label{eq:cycconsfm}
\end{align}
Analogously to \eqref{eq:stackedU}, we stack all $\fm_i$ into a tall $(k\dimLb{\times}\dimLb)$-dimensional block-matrix that we call
\begin{align}
    Q = \begin{bmatrix} \fm_1^\top, \fm_2^\top, \ldots \fm_k^\top \end{bmatrix}^\top\,.
    \label{eq:stackedQ}
\end{align}

In accordance with the definition of the permutation constraint, we define the stacked block-orthogonal constraint $Q \in \orth$ (without subscript in $\orth$) that indicates that every block $\fm_i \in \orth_b$.  

\section{Isometric Multi-Shape Matching}
\begin{table}[t]
\footnotesize
\begin{tabular}{@{}lp{5.4cm}@{}}
\toprule
\textbf{Symbol} & \textbf{Meaning} \\ \midrule
$k$ & total number of shapes to be matched \\
$m_i$ & total number of points in shape $i$ \\
$m = \sum_{i=1}^k m_i$ & total number of points across all $k$ shapes \\
$d$ & universe size (total number of unique points across all shapes) \\
$P_i \in \perm_{m_id} \subset \R^{m_i {\times} d}$ & shape-to-universe matching for shape $i$ \\
$U \in \perm \subset \R^{m \times d}$ & stack of all shape-to-universe matchings \\
$b$ & number of LBO basis functions \\
$\Phi_i \in \R^{m_i \times b}$ & eigenfunction of the LBO of shape $i$ \\
 $\mathbf{\Phi} \in \R^{m \times kb}$ & block-diagonal matrix containing the eigenfunctions of all shapes \\
$\fm_i\in \R^{b \times b}$ & shape-to-universe functional map  for shape $i$ \\
$Q \in \orth \subset \R^{kb \times b}$ & stack of all shape-to-universe functional maps \\
 \bottomrule
\end{tabular}
\caption{Overview of our notation.}
\label{table:notation}
\end{table}

In this section, we introduce our matching formulation, the optimisation algorithm thereof, and provide a theoretical analysis. Our notation is summarised in Tab.~\ref{table:notation}.

\subsection{Problem Formulation}
The objective function of our isometric multi-matching formulation (that we will later maximise) reads
\begin{align}
f(U,Q) &= \sum_{i,j=1}^k \langle P_i^\top \Phi_i \fm_i, P_j^\top \Phi_j \fm_j \rangle  \label{eq:objSum} \\
        &= \langle U^\top \mathbf{\Phi} Q , U^\top \mathbf{\Phi} Q \rangle \,, \label{eq:objMatrix}
\end{align}
where $ \mathbf{\Phi} = \diag(\Phi_1, \ldots, \Phi_k) \in \R^{m \times k\dimLb}$.
The equality between the explicit summation formulation in \eqref{eq:objSum} and the matrix formulation in \eqref{eq:objMatrix} can be verified by expanding the matrix multiplications.
When maximising the objective function, the inner product between the aligned basis functions $\Phi_i$ and $\Phi_j$ is maximised for all pairs $i,j$. 
For that purpose, $P_i$ and $P_j$ permute the vertices in terms of universe points, while $\fm_i$ and $\fm_j$ align the basis functions on the same universe points via an orthogonal transform. 
Rewriting each summand of \eqref{eq:objSum} as $\textbf{tr}((P_i^\top \Phi_i) \fm_i \fm_j^\top (P_j^\top \Phi_j)^\top)$, we see each operation explicitly: 
$P_\bullet^\top \Phi_\bullet$ shuffles the vertices into consistent universe ordering, $\fm_i\fm_j$ composes the (cycle-consistent) functional maps between $i$ and $j$ according to \eqref{eq:cycconsfm}.

The overall optimisation is performed with respect to $U$ and $Q$, with the constraints $U \in \perm$ and $Q  \in \orth$. 
As such, our isometric multi-shape matching formulation reads
\begin{align}
    \max_{U,Q } & &  \langle U^\top \mathbf{\Phi} Q , U^\top \mathbf{\Phi} Q \rangle \label{eq:optprob}\\
    \text{s.t.}&& U \in \perm, Q  \in \orth \,.\nonumber   
\end{align}

\subsection{Algorithm}
In order to solve Problem~\eqref{eq:optprob}, we propose a novel projection-based algorithm that we call \textsc{IsoMuSh} (\textbf{Iso}metric \textbf{Mu}lti-\textbf{Sh}ape Matching). 
The optimisation alternates between updating $U$ and $Q$. 
Each update step involves simple matrix multiplications, as well as the Euclidean projection onto the sets $\perm$ and $\orth$. For permutations, as well as different objective functions, a similar strategy has been proven effective in~\cite{vestner2017efficient,bernard2019hippi}.
We denote the Euclidean projections as $\proj_{\perm}(\cdot)$ and $\proj_{\orth}(\cdot)$.
Each Euclidean projection returns the closest element in the constraint set according to the squared Frobenius norm. For the set $\orth$, it is defined as
\begin{align}
\proj_\orth(Q) &= \argmin_{Y\in \orth} \Vert Q - Y \Vert_F^2 \\
&= \argmax_{Y \in \orth} \,2\langle Q, Y \rangle {-} \langle Y, Y \rangle = \argmax_{Y \in \orth} \,\langle Q, Y \rangle \,. \nonumber
\end{align} 
The last equality arises from the orthonormality of all $\fm_i$ in $Q$. The projection onto the set $\perm$ is defined analogously, in which case the term $\langle Y,Y \rangle$ has the constant value $m$ for $Y \in \perm$ (since the term simply counts the total number of ones in $Y$, which has the fixed value $m$ because $U \in \perm$ implies $U\onevec_d = \onevec_m$).  By $U_t$ and $Q_t$ we denote the values of $U$ and $Q$ at iteration $t$, respectively.

\begin{algorithm}[t]\label{alg:isomush}
\SetKwInput{Input}{Input}
\SetKwInput{Output}{Output}
\SetKwInput{Initialise}{Initialise}
\SetKwRepeat{Do}{do}{while}
\DontPrintSemicolon
 \Input{$\mathbf{\Phi}$, $\epsilon$ (relative objective improvement)}
 \Output{$U,Q$}
\Initialise{$t \gets 0, U_0 \in \perm, Q_0 \in \orth$}
  \Repeat{ $ \frac{f(U_t,Q_t)}{f(U_{t+1},Q_{t+1})} \geq 1{-}\epsilon$}{ 
    $U_{t{+}1} \gets \proj_{\perm}(\mathbf{\Phi}Q_tQ_t^\top \mathbf{\Phi}^\top U_t  )$\\
    $Q_{t{+}1} \gets \proj_{\orth}(\mathbf{\Phi}^\top U_{t{+}1} U_{t{+}1}^\top \mathbf{\Phi}Q_t)$\\
    $t \gets t{+}1$
   }
 \caption{\textsc{IsoMuSh} algorithm.}
\end{algorithm}

\paragraph{$U$-update. } 
For $Z= \mathbf{\Phi} Q_t Q_t^\top \mathbf{\Phi}^\top$, the $U$-update step projects 
$ZU_t$ onto $\perm$. Hence, the $U$-update reads
\begin{align}
    U_{t+1} = \proj_{\perm}(Z U_t) &= \argmax_{U \in \perm}~\langle Z U_t, U  \rangle \label{eq:projP} \\
    &= 
    \begin{bmatrix}
        \argmax\limits_{ P_1 \in \perm_{m_1 d}  }~ \langle [ZU_t]_1, P_1 \rangle  \\
        \vdots \\
        \argmax\limits_{ P_k \in \perm_{m_k d}  } ~\langle [Z U_t]_k, P_k \rangle 
    \end{bmatrix}\,,
    \label{eq:sepLAPs}
\end{align}
where $[ZU_t]_i$ denotes the $i$-th block (of size $m_i {\times} d$) of $ZU_t$.
Each block of $U$ in \eqref{eq:projP} is independent, and consequently can be optimised for separately, as written in~\eqref{eq:sepLAPs}. 
This reduces the projection into solving $k$ independent (partial) linear assignment problems. To this end, we use an efficient implementation~\cite{bernard2016fast} of the Auction algorithm~\cite{Bertsekas:1998vt}.

\paragraph{$Q$-update. }
For $\overline{Z} = \mathbf{\Phi}^\top U_{t{+}1} U_{t{+}1}^\top \mathbf{\Phi}$, the $Q$-update step projects $\overline{Z}Q_t$ onto $\orth$. It is given by
\begin{align}
    Q_{t+1} = \proj_{\orth}(\overline{Z} Q_t) &= \argmax_{Q \in \orth} \langle \overline{Z} Q_t, Q  \rangle \label{eq:projO} \\
    &=  \begin{bmatrix}
       \argmax\limits_{\fm_1 \in \orth} \langle [\overline{Z}Q_t]_1, \fm_1  \rangle  \\
        \vdots \\
        \argmax\limits_{\fm_k \in \orth} \langle [\overline{Z}Q_t]_k, \fm_k  \rangle
    \end{bmatrix}\,, \label{eq:projOblock}
\end{align}
where $[\overline{Z}Q_t]_i$ denotes the $i$-th block (of size $b{\times}b$) of $\overline{Z}Q_t$.
Similar as in the $U$-update, the result for each block of $Q$ in \eqref{eq:projO} is independent, and can thus be optimised separately, as shown in~\eqref{eq:projOblock}. 
Therefore, we can solve $k$ independent singular value decompositions (SVDs), each for a small matrix of size $b{\times}b$.

\subsection{Theoretical Analysis}
In this section, the properties of the \textsc{IsoMuSh} algorithm is analysed. To this end, we prove that the algorithm convergences, and present a complexity analysis.

\subsubsection{Convergence} 
The convergence of  our algorithm follows from the monotonicity of the individual updates. Here, we present the respective results, and refer readers to the supplementary material for the proofs.

\begin{lem}\label{lem:U}
$\langle U_{t}^\top \mathbf{\Phi} Q_t , U_{t+1}^\top \mathbf{\Phi} Q_t \rangle \geq \langle U_t^\top \mathbf{\Phi} Q_t , U_t^\top \mathbf{\Phi} Q_t \rangle$ holds for any $t$.
\end{lem}

\begin{prop}[Monotonicity of $U$-update]\label{prop:U}\ \\
The objective value cannot decrease through the $U$-update step~\eqref{eq:projP}, and $\langle U_{t+1}^\top \mathbf{\Phi} Q_t , U_{t+1}^\top \mathbf{\Phi} Q_t \rangle \geq \langle U_t^\top \mathbf{\Phi} Q_t , U_t^\top \mathbf{\Phi} Q_t \rangle$ holds.
\end{prop}

\begin{lem}\label{lem:Q} In each iteration $t$,
$\langle U_{t+1}^\top \mathbf{\Phi} Q_t , U_{t+1}^\top \mathbf{\Phi} Q_{t+1} \rangle \geq \langle U_{t+1}^\top \mathbf{\Phi} Q_t , U_{t+1}^\top \mathbf{\Phi} Q_t \rangle$ holds.
\end{lem}

\begin{prop}[Monotonicity of $Q$-update]\label{prop:Q} \ \\
The objective value cannot decrease through the $Q$-update~\eqref{eq:projO}, and $\langle U_{t+1}^\top \mathbf{\Phi} Q_{t+1} , U_{t+1}^\top \mathbf{\Phi} Q_{t+1} \rangle \geq \langle U_{t+1}^\top \mathbf{\Phi} Q_t , U_{t+1}^\top \mathbf{\Phi} Q_t \rangle$ holds.
\end{prop}
By combining these properties, and exploiting that $U$ and $Q$ are in compact sets, we obtain the following result:
\begin{theorem}[Convergence]\ \\
The sequence $(f(U_t,Q_t))_{t=1,2,\ldots}$ is monotonically increasing and convergent. Algorithm~\ref{alg:isomush} terminates in finite time.
\end{theorem}

\subsubsection{Complexity Analysis}
The steps in the \textsc{IsoMuSh} algorithm comprises matrix multiplications and projections onto the sets $\perm$ and $\orth$. In the following, we break down the complexity of each step:

\paragraph{Multiplications in $U$-update:} The term $\mathbf{\Phi}QQ^\top \mathbf{\Phi}^\top U$ can be computed as $AB$ for $A=\mathbf{\Phi}Q$ and $B = A^\top U$. 
Computing    $A \in \R^{m \times b}$ has complexity $\mathcal{O}(b^2m)$ ($\mathbf{\Phi}$ is a block-diagonal matrix). Computing $B = A^\top U \in \R^{b \times d}$ has complexity $\mathcal{O}(bdk)$ ($U$ is a sparse matrix with at most $k$ nonzero elements per column). Finally, computing $AB \in \R^{m \times d}$ has complexity  $\mathcal{O}(bdm)$. This results in an complexity of $\mathcal{O}(bm \cdot \max(d, b))$ for the U-step matrix multiplication.

\paragraph{Multiplications in $Q$-update:} The term $\mathbf{\Phi}^\top UU^\top \mathbf{\Phi}Q$ can be computed as $C^\top D$ for $C=U^\top\mathbf{\Phi}$ and $D = CQ$.
Computing  $C \in \R^{d \times kb}$ has complexity $\mathcal{O}(bdk)$ ($U$ is a sparse matrix with at most $k$ nonzero elements per column, and $\mathbf{\Phi}$ is a block-diagonal matrix). Computing $D = CQ \in \R^{d \times b}$ has complexity $\mathcal{O}(b^2dk)$. Computing $C^\top D \in \R^{kb \times b}$ has complexity  $\mathcal{O}(b^2dk)$. This results in an complexity of $\mathcal{O}(b^2dk)$ for the Q-step matrix multiplication.

\paragraph{Projection onto $\perm$:} the projection onto $\perm$ is computed by solving $k$ linear assignment problems, each of size $m_i \times d$. The auction algorithm has an average time complexity of (roughly) $\mathcal{O}(d^2 \log(d))$, so that the overall projection leads to $\mathcal{O}(kd^2 \log(d))$.
    
\paragraph{Projection onto $\orth$:} the projection onto $\orth$ is computed by solving $k$ independent projections onto $\orth_b$. Using SVD, this amounts to a complexity of $\mathcal{O}(b^3)$.

\begin{figure*}[t]
    \centering
    {\footnotesize
    \definecolor{mycolor1}{rgb}{0.92900,0.69400,0.12500}%
\definecolor{mycolor2}{rgb}{0.85000,0.32500,0.09800}%
\definecolor{mycolor3}{rgb}{0.49400,0.18400,0.55600}%
\definecolor{mycolor4}{rgb}{0.00000,0.44700,0.74100}%
\begin{tikzpicture}

\begin{axis}[%
width=.25\linewidth,
height=.18\linewidth,
at={(0.797in,0.617in)},
scale only axis,
xmin=0,
xmax=0.15,
ymin=65,
ymax=100,
xmajorgrids,
ymajorgrids,
axis background/.style={fill=white},
xlabel style={font=\color{white!15!black}},
xlabel={Geodesic error},
ylabel style={font=\color{white!15!black}},
ylabel={\% Correspondences},
y label style={at={(axis description cs:0.1,.5)},rotate=0,anchor=south},
title style={font=\bfseries},
title={TOSCA},
legend style={at={(0.97,0.03)}, anchor=south east, legend cell align=left, align=left, draw=white!15!black, font=\footnotesize}
]

\addplot [color=mycolor1, dotted, line width=2.0pt]
  table[row sep=crcr]{%
0	64.5507246376812\\
0.005	66.6086956521739\\
0.01	71.8840579710145\\
0.015	75.8695652173913\\
0.02	78.9855072463768\\
0.025	81.2898550724638\\
0.03	83.1449275362319\\
0.035	84.6811594202899\\
0.04	85.8695652173913\\
0.045	87.1014492753623\\
0.05	88.2463768115942\\
0.055	89.1014492753623\\
0.06	89.9130434782609\\
0.065	90.4057971014493\\
0.07	91.0289855072464\\
0.075	91.5797101449275\\
0.08	91.9710144927536\\
0.085	92.304347826087\\
0.09	92.7536231884058\\
0.095	93.0724637681159\\
0.1	93.3188405797101\\
0.105	93.5652173913043\\
0.11	93.7826086956522\\
0.115	93.9275362318841\\
0.12	94.1014492753623\\
0.125	94.2463768115942\\
0.13	94.2753623188406\\
0.135	94.5217391304348\\
0.14	94.6521739130435\\
0.145	94.768115942029\\
0.15	94.8260869565217\\
0.155	94.9420289855072\\
0.16	95.0434782608696\\
0.165	95.1014492753623\\
0.17	95.1594202898551\\
0.175	95.2173913043478\\
0.18	95.2173913043478\\
0.185	95.304347826087\\
0.19	95.3768115942029\\
0.195	95.3913043478261\\
0.2	95.4492753623188\\
0.205	95.5072463768116\\
0.21	95.5217391304348\\
0.215	95.5797101449275\\
0.22	95.6376811594203\\
0.225	95.6521739130435\\
0.23	95.6811594202899\\
0.235	95.7246376811594\\
0.24	95.768115942029\\
0.245	95.768115942029\\
0.25	95.8260869565217\\
};
\addlegendentry{ZoomOut}

\addplot [color=mycolor2, dashed, line width=2.0pt]
  table[row sep=crcr]{%
0	80.7826086956522\\
0.005	81.5507246376812\\
0.01	83.3478260869565\\
0.015	85.0579710144928\\
0.02	87.0579710144928\\
0.025	88.9420289855072\\
0.03	90.4492753623188\\
0.035	91.5942028985507\\
0.04	92.5942028985507\\
0.045	93.6811594202899\\
0.05	94.768115942029\\
0.055	95.3478260869565\\
0.06	95.8840579710145\\
0.065	96.4202898550725\\
0.07	96.6811594202899\\
0.075	96.7826086956522\\
0.08	97\\
0.085	97.1594202898551\\
0.09	97.2753623188406\\
0.095	97.3478260869565\\
0.1	97.4782608695652\\
0.105	97.6231884057971\\
0.11	97.7246376811594\\
0.115	97.8550724637681\\
0.12	97.9275362318841\\
0.125	98.0144927536232\\
0.13	98.0869565217391\\
0.135	98.1449275362319\\
0.14	98.231884057971\\
0.145	98.3333333333333\\
0.15	98.4057971014493\\
0.155	98.4202898550725\\
0.16	98.4782608695652\\
0.165	98.5217391304348\\
0.17	98.5942028985507\\
0.175	98.6521739130435\\
0.18	98.7101449275362\\
0.185	98.768115942029\\
0.19	98.8550724637681\\
0.195	98.8985507246377\\
0.2	98.9275362318841\\
0.205	98.9565217391304\\
0.21	98.9710144927536\\
0.215	99.0144927536232\\
0.22	99.0434782608696\\
0.225	99.0434782608696\\
0.23	99.0869565217391\\
0.235	99.1304347826087\\
0.24	99.1304347826087\\
0.245	99.1449275362319\\
0.25	99.1739130434783\\
};
\addlegendentry{ZoomOut+Sync}

\addplot [color=mycolor3, line width=2.0pt]
  table[row sep=crcr]{%
0	74.3623188405797\\
0.005	76.8260869565217\\
0.01	82.0724637681159\\
0.015	85.8695652173913\\
0.02	88.8550724637681\\
0.025	90.8840579710145\\
0.03	92.3478260869565\\
0.035	93.8115942028985\\
0.04	95.0724637681159\\
0.045	96\\
0.05	96.7101449275362\\
0.055	97.463768115942\\
0.06	97.9855072463768\\
0.065	98.3333333333333\\
0.07	98.6666666666667\\
0.075	98.9130434782609\\
0.08	99.0434782608696\\
0.085	99.1594202898551\\
0.09	99.304347826087\\
0.095	99.3623188405797\\
0.1	99.4347826086957\\
0.105	99.5072463768116\\
0.11	99.536231884058\\
0.115	99.5942028985507\\
0.12	99.6086956521739\\
0.125	99.6666666666667\\
0.13	99.6811594202899\\
0.135	99.7536231884058\\
0.14	99.7826086956522\\
0.145	99.8115942028985\\
0.15	99.8115942028985\\
0.155	99.8405797101449\\
0.16	99.8550724637681\\
0.165	99.8550724637681\\
0.17	99.8695652173913\\
0.175	99.8695652173913\\
0.18	99.8840579710145\\
0.185	99.9565217391304\\
0.19	99.9855072463768\\
0.195	99.9855072463768\\
0.2	99.9855072463768\\
0.205	100\\
0.21	100\\
0.215	100\\
0.22	100\\
0.225	100\\
0.23	100\\
0.235	100\\
0.24	100\\
0.245	100\\
0.25	100\\
};
\addlegendentry{ConsistentZoomOut}

\addplot [color=mycolor4, line width=2.0pt]
  table[row sep=crcr]{%
0	83.4057971014493\\
0.005	83.7826086956522\\
0.01	85.4347826086957\\
0.015	87.1159420289855\\
0.02	88.8260869565217\\
0.025	90.463768115942\\
0.03	91.8840579710145\\
0.035	92.768115942029\\
0.04	93.6811594202899\\
0.045	94.5797101449275\\
0.05	95.536231884058\\
0.055	95.9565217391304\\
0.06	96.3913043478261\\
0.065	96.8695652173913\\
0.07	97.1304347826087\\
0.075	97.2463768115942\\
0.08	97.463768115942\\
0.085	97.6231884057971\\
0.09	97.7391304347826\\
0.095	97.768115942029\\
0.1	97.8985507246377\\
0.105	98.0579710144928\\
0.11	98.1304347826087\\
0.115	98.231884057971\\
0.12	98.2898550724638\\
0.125	98.3768115942029\\
0.13	98.463768115942\\
0.135	98.5217391304348\\
0.14	98.6086956521739\\
0.145	98.695652173913\\
0.15	98.7246376811594\\
0.155	98.7391304347826\\
0.16	98.768115942029\\
0.165	98.8115942028985\\
0.17	98.8985507246377\\
0.175	98.9130434782609\\
0.18	98.9565217391304\\
0.185	99.0144927536232\\
0.19	99.0869565217391\\
0.195	99.1304347826087\\
0.2	99.1739130434783\\
0.205	99.2173913043478\\
0.21	99.2463768115942\\
0.215	99.2898550724638\\
0.22	99.3188405797101\\
0.225	99.3188405797101\\
0.23	99.3623188405797\\
0.235	99.4202898550725\\
0.24	99.4202898550725\\
0.245	99.4347826086957\\
0.25	99.463768115942\\
};
\addlegendentry{HiPPI}

\addplot [color=red, line width=2.0pt]
  table[row sep=crcr]{%
0	87.768115942029\\
0.005	88.3188405797101\\
0.01	89.5797101449275\\
0.015	90.8695652173913\\
0.02	92.0289855072464\\
0.025	93.1884057971015\\
0.03	94.2463768115942\\
0.035	94.8405797101449\\
0.04	95.4782608695652\\
0.045	96.3478260869565\\
0.05	96.9855072463768\\
0.055	97.4347826086957\\
0.06	97.7536231884058\\
0.065	98.0724637681159\\
0.07	98.3333333333333\\
0.075	98.4347826086957\\
0.08	98.6086956521739\\
0.085	98.768115942029\\
0.09	98.8985507246377\\
0.095	98.9420289855072\\
0.1	98.9565217391304\\
0.105	99.0724637681159\\
0.11	99.1449275362319\\
0.115	99.2028985507246\\
0.12	99.3333333333333\\
0.125	99.4347826086957\\
0.13	99.4782608695652\\
0.135	99.5217391304348\\
0.14	99.5507246376812\\
0.145	99.6376811594203\\
0.15	99.6666666666667\\
0.155	99.695652173913\\
0.16	99.7391304347826\\
0.165	99.7536231884058\\
0.17	99.7826086956522\\
0.175	99.7826086956522\\
0.18	99.8260869565217\\
0.185	99.8695652173913\\
0.19	99.8840579710145\\
0.195	99.9130434782609\\
0.2	99.9275362318841\\
0.205	99.9565217391304\\
0.21	99.9565217391304\\
0.215	99.9565217391304\\
0.22	99.9565217391304\\
0.225	99.9565217391304\\
0.23	99.9565217391304\\
0.235	99.9565217391304\\
0.24	99.9565217391304\\
0.245	99.9710144927536\\
0.25	99.9710144927536\\
};
\addlegendentry{Ours}

\end{axis}
\end{tikzpicture}%
    \definecolor{mycolor1}{rgb}{0.92900,0.69400,0.12500}%
\definecolor{mycolor2}{rgb}{0.85000,0.32500,0.09800}%
\definecolor{mycolor3}{rgb}{0.49400,0.18400,0.55600}%
\definecolor{mycolor4}{rgb}{0.00000,0.44700,0.74100}%
\begin{tikzpicture}

\begin{axis}[%
width=.25\linewidth,
height=.18\linewidth,
at={(2.6in,1.228in)},
scale only axis,
xmin=0,
xmax=0.15,
xlabel style={font=\color{white!15!black}},
xlabel={Geodesic error},
ymin=65,
ymax=100,
xmajorgrids,
ymajorgrids,
ylabel style={font=\color{white!15!black}},
axis background/.style={fill=white},
title style={font=\bfseries},
title={FAUST},
]

\addplot [color=mycolor1, dashed, line width=2.0pt]
  table[row sep=crcr]{%
0	66.9857142857143\\
0.005	70.2285714285714\\
0.01	75.6571428571429\\
0.015	78.4\\
0.02	80.3428571428571\\
0.025	81.8142857142857\\
0.03	83.5571428571429\\
0.035	84.6\\
0.04	85.5\\
0.045	86.5857142857143\\
0.05	87.5\\
0.055	88.4571428571429\\
0.06	89.2285714285714\\
0.065	90.1285714285714\\
0.07	90.8285714285714\\
0.075	91.4285714285714\\
0.08	92.0142857142857\\
0.085	92.5571428571429\\
0.09	93.0428571428571\\
0.095	93.6571428571429\\
0.1	94.0857142857143\\
0.105	94.4142857142857\\
0.11	94.8\\
0.115	95.2\\
0.12	95.5714285714286\\
0.125	95.8571428571429\\
0.13	96.1142857142857\\
0.135	96.4571428571429\\
0.14	96.7571428571429\\
0.145	96.9428571428571\\
0.15	97.1571428571429\\
0.155	97.4428571428571\\
0.16	97.6285714285714\\
0.165	97.7571428571429\\
0.17	97.9142857142857\\
0.175	98.0428571428571\\
0.18	98.1714285714286\\
0.185	98.2285714285714\\
0.19	98.3714285714286\\
0.195	98.4285714285714\\
0.2	98.5857142857143\\
0.205	98.6714285714286\\
0.21	98.7142857142857\\
0.215	98.8285714285714\\
0.22	98.8428571428571\\
0.225	98.8571428571429\\
0.23	98.9714285714286\\
0.235	98.9714285714286\\
0.24	98.9857142857143\\
0.245	99.0142857142857\\
0.25	99.0428571428571\\
};

\addplot [color=mycolor2, dashed, line width=2.0pt]
  table[row sep=crcr]{%
0	76.8857142857143\\
0.005	77.4714285714286\\
0.01	78.7571428571429\\
0.015	80.0857142857143\\
0.02	81.5571428571429\\
0.025	82.9714285714286\\
0.03	84.7857142857143\\
0.035	86.1285714285714\\
0.04	87.1\\
0.045	88.2428571428571\\
0.05	89.3142857142857\\
0.055	90.1142857142857\\
0.06	90.9571428571429\\
0.065	91.6714285714286\\
0.07	92.3142857142857\\
0.075	92.7571428571429\\
0.08	93.3571428571429\\
0.085	93.8714285714286\\
0.09	94.3571428571429\\
0.095	94.8714285714286\\
0.1	95.3142857142857\\
0.105	95.6714285714286\\
0.11	96.0714285714286\\
0.115	96.4571428571429\\
0.12	96.7857142857143\\
0.125	97.0428571428571\\
0.13	97.2142857142857\\
0.135	97.5714285714286\\
0.14	97.8857142857143\\
0.145	98.0714285714286\\
0.15	98.3\\
0.155	98.5428571428571\\
0.16	98.6428571428571\\
0.165	98.7571428571429\\
0.17	98.9285714285714\\
0.175	99.0428571428571\\
0.18	99.2\\
0.185	99.2142857142857\\
0.19	99.2857142857143\\
0.195	99.3571428571429\\
0.2	99.4571428571429\\
0.205	99.4857142857143\\
0.21	99.5428571428571\\
0.215	99.6285714285714\\
0.22	99.6714285714286\\
0.225	99.7428571428571\\
0.23	99.7857142857143\\
0.235	99.8142857142857\\
0.24	99.8571428571429\\
0.245	99.8714285714286\\
0.25	99.8857142857143\\
};

\addplot [color=mycolor3, line width=2.0pt]
  table[row sep=crcr]{%
0	68.9857142857143\\
0.005	72.0857142857143\\
0.01	77.3428571428571\\
0.015	80.4142857142857\\
0.02	82.3428571428571\\
0.025	83.8142857142857\\
0.03	85.3\\
0.035	87.1714285714286\\
0.04	88.0285714285714\\
0.045	89.0285714285714\\
0.05	89.8285714285714\\
0.055	90.5\\
0.06	91.1285714285714\\
0.065	91.7857142857143\\
0.07	92.4142857142857\\
0.075	93\\
0.08	93.6142857142857\\
0.085	94\\
0.09	94.5285714285714\\
0.095	95.0571428571429\\
0.1	95.3857142857143\\
0.105	95.6857142857143\\
0.11	96.0285714285714\\
0.115	96.5571428571429\\
0.12	96.9428571428571\\
0.125	97.1714285714286\\
0.13	97.3571428571429\\
0.135	97.6285714285714\\
0.14	97.9857142857143\\
0.145	98.1428571428571\\
0.15	98.3714285714286\\
0.155	98.6\\
0.16	98.8\\
0.165	98.8857142857143\\
0.17	99.0428571428571\\
0.175	99.1571428571429\\
0.18	99.2857142857143\\
0.185	99.3\\
0.19	99.4714285714286\\
0.195	99.5285714285714\\
0.2	99.6285714285714\\
0.205	99.6714285714286\\
0.21	99.7\\
0.215	99.8\\
0.22	99.8285714285714\\
0.225	99.8571428571429\\
0.23	99.9285714285714\\
0.235	99.9285714285714\\
0.24	99.9428571428571\\
0.245	99.9714285714286\\
0.25	99.9857142857143\\
};

\addplot [color=mycolor4, line width=2.0pt]
  table[row sep=crcr]{%
0	78.2714285714286\\
0.005	78.4142857142857\\
0.01	79.3714285714286\\
0.015	80.6285714285714\\
0.02	82.1\\
0.025	83.4714285714286\\
0.03	85.3\\
0.035	86.4571428571429\\
0.04	87.4285714285714\\
0.045	88.5\\
0.05	89.5571428571429\\
0.055	90.3142857142857\\
0.06	91.0714285714286\\
0.065	91.8714285714286\\
0.07	92.4285714285714\\
0.075	92.9142857142857\\
0.08	93.5142857142857\\
0.085	94\\
0.09	94.5571428571429\\
0.095	95.0571428571429\\
0.1	95.4571428571429\\
0.105	95.8142857142857\\
0.11	96.2\\
0.115	96.6\\
0.12	96.8571428571429\\
0.125	97.0857142857143\\
0.13	97.2571428571429\\
0.135	97.5857142857143\\
0.14	97.9285714285714\\
0.145	98.1142857142857\\
0.15	98.3714285714286\\
0.155	98.5857142857143\\
0.16	98.6857142857143\\
0.165	98.8142857142857\\
0.17	98.9714285714286\\
0.175	99.0857142857143\\
0.18	99.2428571428571\\
0.185	99.2571428571429\\
0.19	99.3285714285714\\
0.195	99.4\\
0.2	99.5\\
0.205	99.5285714285714\\
0.21	99.5857142857143\\
0.215	99.6714285714286\\
0.22	99.7142857142857\\
0.225	99.7714285714286\\
0.23	99.8142857142857\\
0.235	99.8428571428571\\
0.24	99.8857142857143\\
0.245	99.8857142857143\\
0.25	99.9\\
};

\addplot [color=red, line width=2.0pt]
  table[row sep=crcr]{%
0	78.6285714285714\\
0.005	79.1428571428571\\
0.01	80.2857142857143\\
0.015	81.5285714285714\\
0.02	82.7714285714286\\
0.025	84.0714285714286\\
0.03	85.6142857142857\\
0.035	86.7285714285714\\
0.04	87.7142857142857\\
0.045	88.6857142857143\\
0.05	89.6571428571429\\
0.055	90.4142857142857\\
0.06	91.3\\
0.065	91.9571428571429\\
0.07	92.5857142857143\\
0.075	93.1\\
0.08	93.7571428571429\\
0.085	94.3\\
0.09	94.7714285714286\\
0.095	95.2571428571429\\
0.1	95.6857142857143\\
0.105	95.9714285714286\\
0.11	96.2857142857143\\
0.115	96.6714285714286\\
0.12	96.9571428571429\\
0.125	97.1571428571429\\
0.13	97.3285714285714\\
0.135	97.6714285714286\\
0.14	97.9142857142857\\
0.145	98.1\\
0.15	98.3714285714286\\
0.155	98.5857142857143\\
0.16	98.6857142857143\\
0.165	98.7857142857143\\
0.17	98.9714285714286\\
0.175	99.1\\
0.18	99.2428571428571\\
0.185	99.2428571428571\\
0.19	99.3142857142857\\
0.195	99.4\\
0.2	99.5142857142857\\
0.205	99.5428571428571\\
0.21	99.6\\
0.215	99.7\\
0.22	99.7428571428571\\
0.225	99.8\\
0.23	99.8571428571429\\
0.235	99.8857142857143\\
0.24	99.9285714285714\\
0.245	99.9285714285714\\
0.25	99.9428571428571\\
};

\end{axis}
\end{tikzpicture}%
    \definecolor{mycolor1}{rgb}{0.92900,0.69400,0.12500}%
\definecolor{mycolor2}{rgb}{0.85000,0.32500,0.09800}%
\definecolor{mycolor3}{rgb}{0.49400,0.18400,0.55600}%
\definecolor{mycolor4}{rgb}{0.00000,0.44700,0.74100}%
\begin{tikzpicture}

\begin{axis}[%
width=.25\linewidth,
height=.18\linewidth,
at={(2.6in,1.228in)},
scale only axis,
xmin=0,
xmax=0.15,
xlabel style={font=\color{white!15!black}},
xlabel={Geodesic error},
ymin=65,
ymax=100,
xmajorgrids,
ymajorgrids,
ylabel style={font=\color{white!15!black}},
axis background/.style={fill=white},
title style={font=\bfseries},
title={SCAPE},
]

\addplot [color=mycolor1, dashed, line width=2.0pt]
  table[row sep=crcr]{%
0	63.8666666666667\\
0.005	64.2333333333333\\
0.01	67.2333333333333\\
0.015	71.5666666666667\\
0.02	76.4666666666667\\
0.025	79.0333333333333\\
0.03	80.5666666666667\\
0.035	83.1666666666667\\
0.04	85.1333333333333\\
0.045	86.5333333333333\\
0.05	88.2\\
0.055	89.1666666666667\\
0.06	90.4666666666667\\
0.065	91.3\\
0.07	91.8333333333333\\
0.075	92.1333333333333\\
0.08	92.5666666666667\\
0.085	93.0666666666667\\
0.09	93.4666666666667\\
0.095	93.9666666666667\\
0.1	94.5333333333333\\
0.105	94.6666666666667\\
0.11	95.2666666666667\\
0.115	95.7333333333333\\
0.12	96.0666666666667\\
0.125	96.5\\
0.13	96.7333333333333\\
0.135	97.1333333333333\\
0.14	97.5333333333333\\
0.145	97.7\\
0.15	98\\
0.155	98.1333333333333\\
0.16	98.3333333333333\\
0.165	98.5666666666667\\
0.17	98.6333333333333\\
0.175	98.7666666666667\\
0.18	98.9666666666667\\
0.185	99.1\\
0.19	99.2333333333333\\
0.195	99.4\\
0.2	99.4666666666667\\
0.205	99.6333333333333\\
0.21	99.7666666666667\\
0.215	99.9333333333333\\
0.22	99.9333333333333\\
0.225	99.9333333333333\\
0.23	99.9333333333333\\
0.235	99.9666666666667\\
0.24	100\\
0.245	100\\
0.25	100\\
};

\addplot [color=mycolor2, dashed, line width=2.0pt]
  table[row sep=crcr]{%
0	74.1333333333333\\
0.005	74.4666666666667\\
0.01	75.9\\
0.015	78.0666666666667\\
0.02	81.8\\
0.025	84.1\\
0.03	86.7333333333333\\
0.035	88.9666666666667\\
0.04	90.7333333333333\\
0.045	92\\
0.05	93.1666666666667\\
0.055	93.9666666666667\\
0.06	94.9\\
0.065	95.5\\
0.07	95.9333333333333\\
0.075	96.2666666666667\\
0.08	96.3666666666667\\
0.085	96.5333333333333\\
0.09	96.9333333333333\\
0.095	97.2666666666667\\
0.1	97.3666666666667\\
0.105	97.4333333333333\\
0.11	97.5666666666667\\
0.115	97.6666666666667\\
0.12	97.8666666666667\\
0.125	98.1333333333333\\
0.13	98.2333333333333\\
0.135	98.3666666666667\\
0.14	98.5\\
0.145	98.7333333333333\\
0.15	98.9\\
0.155	99.0333333333333\\
0.16	99.1\\
0.165	99.2333333333333\\
0.17	99.3\\
0.175	99.3666666666667\\
0.18	99.3666666666667\\
0.185	99.3666666666667\\
0.19	99.4666666666667\\
0.195	99.5\\
0.2	99.5\\
0.205	99.6\\
0.21	99.6\\
0.215	99.6333333333333\\
0.22	99.6333333333333\\
0.225	99.6333333333333\\
0.23	99.6333333333333\\
0.235	99.6333333333333\\
0.24	99.6666666666667\\
0.245	99.7\\
0.25	99.7\\
};

\addplot [color=mycolor3, line width=2.0pt]
  table[row sep=crcr]{%
0	68.3666666666667\\
0.005	68.8333333333333\\
0.01	72.1333333333333\\
0.015	76.5666666666667\\
0.02	81\\
0.025	83.8\\
0.03	86.0333333333333\\
0.035	88.2666666666667\\
0.04	90.2333333333333\\
0.045	91.7333333333333\\
0.05	93.1666666666667\\
0.055	94.5\\
0.06	95\\
0.065	95.8\\
0.07	96.2666666666667\\
0.075	96.5333333333333\\
0.08	96.8666666666667\\
0.085	97\\
0.09	97.2\\
0.095	97.4666666666667\\
0.1	97.6333333333333\\
0.105	97.7333333333333\\
0.11	97.9666666666667\\
0.115	98.0666666666667\\
0.12	98.2\\
0.125	98.2666666666667\\
0.13	98.5666666666667\\
0.135	98.6666666666667\\
0.14	98.8\\
0.145	98.9666666666667\\
0.15	99.1666666666667\\
0.155	99.2666666666667\\
0.16	99.3\\
0.165	99.3666666666667\\
0.17	99.4333333333333\\
0.175	99.5\\
0.18	99.5\\
0.185	99.5333333333333\\
0.19	99.5666666666667\\
0.195	99.6666666666667\\
0.2	99.7333333333333\\
0.205	99.7666666666667\\
0.21	99.8\\
0.215	99.9\\
0.22	99.9333333333333\\
0.225	99.9333333333333\\
0.23	99.9333333333333\\
0.235	99.9333333333333\\
0.24	99.9666666666667\\
0.245	99.9666666666667\\
0.25	100\\
};

\addplot [color=mycolor4, line width=2.0pt]
  table[row sep=crcr]{%
0	78.6666666666667\\
0.005	78.8333333333333\\
0.01	80.1333333333333\\
0.015	82.2\\
0.02	85.0333333333333\\
0.025	87.0666666666667\\
0.03	88.8333333333333\\
0.035	90.6333333333333\\
0.04	91.7666666666667\\
0.045	92.9\\
0.05	94.0333333333333\\
0.055	94.7333333333333\\
0.06	95.6666666666667\\
0.065	96.0666666666667\\
0.07	96.4333333333333\\
0.075	96.7333333333333\\
0.08	96.8666666666667\\
0.085	97.0333333333333\\
0.09	97.3666666666667\\
0.095	97.7\\
0.1	97.8333333333333\\
0.105	97.9333333333333\\
0.11	98.1\\
0.115	98.2333333333333\\
0.12	98.4666666666667\\
0.125	98.6666666666667\\
0.13	98.7333333333333\\
0.135	98.8666666666667\\
0.14	99\\
0.145	99.2\\
0.15	99.3\\
0.155	99.4333333333333\\
0.16	99.5333333333333\\
0.165	99.6\\
0.17	99.6666666666667\\
0.175	99.7\\
0.18	99.7\\
0.185	99.7\\
0.19	99.8\\
0.195	99.8333333333333\\
0.2	99.8333333333333\\
0.205	99.9333333333333\\
0.21	99.9333333333333\\
0.215	99.9666666666667\\
0.22	99.9666666666667\\
0.225	99.9666666666667\\
0.23	99.9666666666667\\
0.235	99.9666666666667\\
0.24	100\\
0.245	100\\
0.25	100\\
};

\addplot [color=red, line width=2.0pt]
  table[row sep=crcr]{%
0	80.2\\
0.005	80.5666666666667\\
0.01	81.7333333333333\\
0.015	83\\
0.02	85.2666666666667\\
0.025	87.0333333333333\\
0.03	89.7666666666667\\
0.035	91.1\\
0.04	92.2\\
0.045	93.7\\
0.05	94.6666666666667\\
0.055	95.3333333333333\\
0.06	96.1333333333333\\
0.065	96.4\\
0.07	96.7666666666667\\
0.075	96.9666666666667\\
0.08	97.0333333333333\\
0.085	97.1\\
0.09	97.4666666666667\\
0.095	97.7666666666667\\
0.1	97.8666666666667\\
0.105	97.9\\
0.11	98.0666666666667\\
0.115	98.3\\
0.12	98.5333333333333\\
0.125	98.8\\
0.13	98.8666666666667\\
0.135	99\\
0.14	99.1\\
0.145	99.3333333333333\\
0.15	99.4333333333333\\
0.155	99.5\\
0.16	99.5666666666667\\
0.165	99.6666666666667\\
0.17	99.7666666666667\\
0.175	99.8\\
0.18	99.8\\
0.185	99.8\\
0.19	99.8333333333333\\
0.195	99.8666666666667\\
0.2	99.8666666666667\\
0.205	99.9333333333333\\
0.21	99.9333333333333\\
0.215	99.9333333333333\\
0.22	99.9333333333333\\
0.225	99.9333333333333\\
0.23	99.9333333333333\\
0.235	99.9333333333333\\
0.24	99.9666666666667\\
0.245	99.9666666666667\\
0.25	99.9666666666667\\
};

\end{axis}
\end{tikzpicture}%
    }
    \caption{Percentage of correct keypoints (PCK) curves for five methods on three datasets, TOSCA, FAUST and SCAPE. Our method leads to better PCK curves (also see the AUC in Tab.~\ref{table:resultSummary}) than its competitors across all datasets. Dashed lines indicate methods that do not jointly optimise for multi-matchings.}
    \label{fig:quantitative}
\end{figure*}
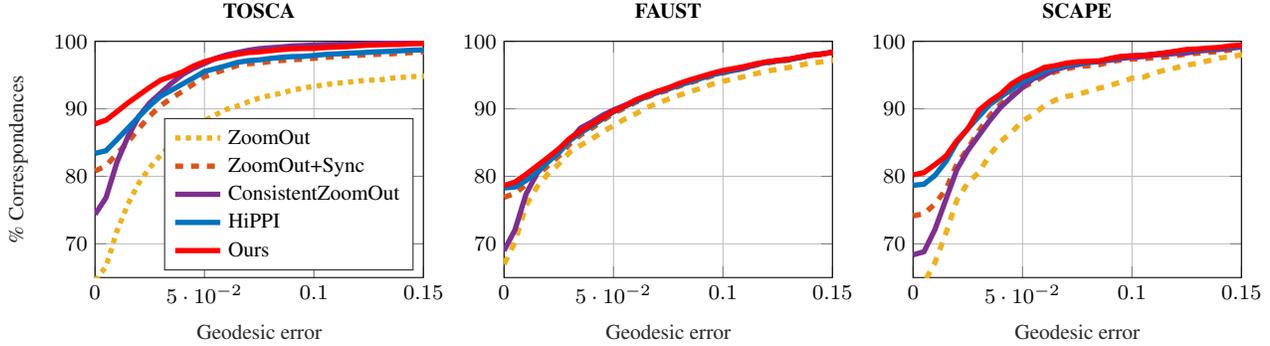

\section{Experiments}\label{sec:exp}

We show the effectiveness of our method on several datasets
and compare against state-of-the-art approaches.

\paragraph{Error measure. }
We evaluate the accuracy of correspondences using the Princeton benchmark protocol \cite{kim2011blended}. 
Given the ground-truth correspondences $(x_i, x_j^*)$ for each $x_i \in \mathcal{X}_i$, the error of the calculated match $(x_i, x_j)$ is given by the normalised geodesic distance between $x_j$ and $x_j^*$
\begin{align}
    \mathbf{e}(x_i) = \frac{\text{dist}_{geo}(x_j, x_j^*)}{\text{diam}(\mathcal{X}_j)}\,,
\end{align}
where $\text{diam}(\cdot)$ denotes the shape diameter.
We plot the accumulated errors smaller than a certain relative error, which is known as \emph{percentage of correct keypoints} (PCK) curve. 
The perfect solution results a constant curve at $100\%$, which amounts to an area under the curve (AUC) of $1$.

\paragraph{Cycle consistency. }
We quantify the cycle consistency of the methods in terms of the \emph{cycle error}, which is the proportion of the number of cycle-consistency violations, divided by the total number of cycles.

\paragraph{Methods.}
We compare our method against several recent state-of-the-art methods, including the pairwise matching approach \textsc{ZoomOut}~\cite{melzi19zoomout}, the two-stage approach \textsc{ZoomOut+Sync} that performs synchronisation to achieve cycle consistency in the results produced by \textsc{ZoomOut}, as well as the multi-matching methods \textsc{HiPPI}~\cite{bernard2019hippi} and \textsc{ConsistentZoomOut}~\cite{huang2020consistent}.

\paragraph{Setup.}
We use results produced by \textsc{ZoomOut} to initialise all other methods.
\textsc{ZoomOut} itself is initialised by the functional map solution~\cite{ovsjanikov2012functional} $\min_{\fm\in\R^{b\times b}} \Vert F\fm - G \Vert_F^2$ (without regularisers), where $F$ and $G$ are the concatenation of normalised Heat Kernel Signature \cite{bronstein2010scale} and \mbox{SHOT \cite{salti2014shot}.}
The output of \textsc{ZoomOut} are pairwise correspondences $\{P_{ij}\}$ and pairwise functional maps $\{\fm_{ij}\}$ between all pairs of shapes.
\textsc{ConsistentZoomOut} directly operates on the $\{\fm_{ij}\}$, so they are used for its initialisation. 
In contrast, \textsc{HiPPI} and our method require shape-to-universe representations. To obtain these, we use  synchronisation to extract the shape-to-universe representation from the pairwise transformations. By doing so, we obtain the initial $U$ and $Q$. We refer to this method of synchronising the \textsc{ZoomOut} results as \textsc{ZoomOut+Sync}, which directly serves as initialisation for \textsc{HiPPI} and our method. Throughout this section we also report results of the initialisation methods \textsc{ZoomOut} and \textsc{ZoomOut+Sync}. Further details can be found in the supplementary material.

\begin{table}[]
\footnotesize
\begin{tabular}{@{}llccccc@{}}
 &  & \multicolumn{1}{l}{\rot{60}{\parbox{1cm}{Ours}}} & \multicolumn{1}{l}{\rot{60}{\parbox{1cm}{\textsc{HiPPI}}}} & \multicolumn{1}{l}{\rot{60}{\parbox{1cm}{\textsc{ZoomOut\\+Sync}}}} & \multicolumn{1}{l}{\rot{60}{\textsc{ZoomOut}}} & \multicolumn{1}{l}{\rot{60}{\parbox{1cm}{\textsc{Consistent\\~~ZoomOut}}}} \\ \toprule
\multirow{3}{*}{\rot{90}{\textbf{TOSCA}}} & AUC $\uparrow$ & \textbf{0.968} & 0.951 & 0.943 & 0.882 & 0.956 \\
 & time [s] $\downarrow$ & \textbf{28.3} & 95.2 & 305.9 & 164.6 & 79.9 \\
 & cycle error $\downarrow$ & \textbf{0} & \textbf{0} & \textbf{0} & 0.68 & 0.17 \\
 \midrule
\multirow{3}{*}{\rot{90}{\textbf{FAUST}}} & AUC $\uparrow$ & \textbf{0.914} & 0.911 & 0.909 & 0.891 & 0.908 \\
 & time [s] $\downarrow$ &\textbf{23.2} & 82.8 & 170.6 & 122.8 & 52.9 \\
 & cycle error $\downarrow$ & \textbf{0} & \textbf{0} & \textbf{0} & 0.41 & 0.16 \\
 \midrule
\multirow{3}{*}{\rot{90}{\textbf{SCAPE}}} & AUC $\uparrow$ & \textbf{0.940} & 0.938 & 0.925 & 0.884 & 0.922 \\
 & time [s] $\downarrow$ & 126.5 & 218.8 & 552.3 & 275.2 & \textbf{82.0} \\
 & cycle error $\downarrow$ & \textbf{0} & \textbf{0} & \textbf{0} & 0.58 & 0.25 \\ \bottomrule
\end{tabular}
\caption{Quantitative evaluation in terms of the area under the PCK curve (AUC), the  runtime (excluding initialisation, which are listed in separate columns), and the cycle error. All values are averaged over all instances for each dataset.}
\label{table:resultSummary}
\end{table}

\subsection{Comparisons to State-of-the-Art Methods}
\paragraph{TOSCA dataset.}
The TOSCA dataset \cite{bronstein2008numerical} contains $76$ shapes from $8$ classes depicting different humans and creatures.
We downsample all shapes to $2{,}000$ faces.
Our method shows state-of-the-art results and surpasses all competitors on this dataset, see Fig.~\ref{fig:quantitative} and Tab.~\ref{table:resultSummary}.
Exemplary matchings of all competing methods can be found in Fig.~\ref{fig:toscaQual}.

\begin{figure*}[h]
\centering
\footnotesize
    \begin{tabular}{c|c|c}
    Colour legend & Ours (bijective \cmark, cycle-consistent \cmark) & \textsc{HiPPI} (bijective \cmark, cycle-consistent \cmark) \\
    \includegraphics[width=.09\linewidth]{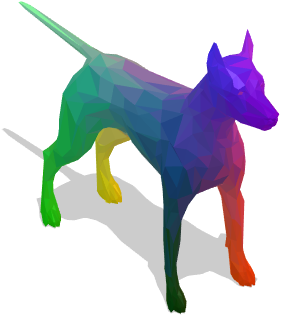} &
    \includegraphics[width=.08\linewidth]{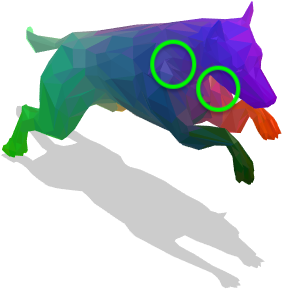} 
    \includegraphics[width=.09\linewidth]{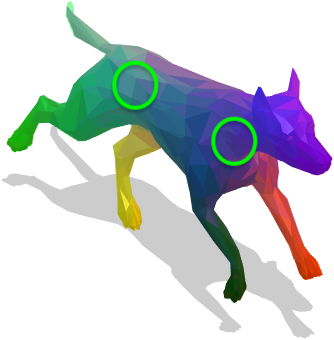} 
    \includegraphics[width=.05\linewidth]{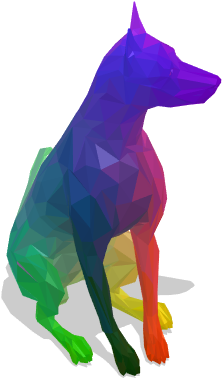} 
    \includegraphics[width=.07\linewidth]{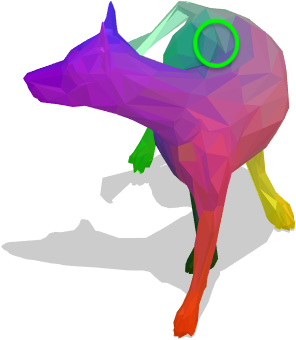} &
    \includegraphics[width=.08\linewidth]{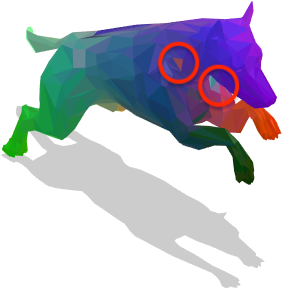} 
    \includegraphics[width=.09\linewidth]{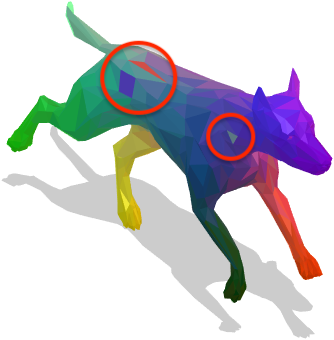} 
    \includegraphics[width=.05\linewidth]{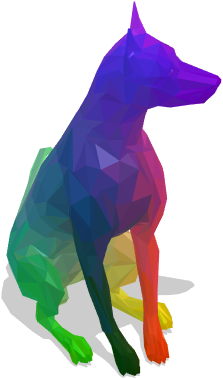} 
    \includegraphics[width=.07\linewidth]{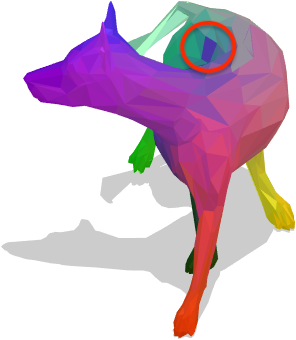} \\
    \hline \vspace{-0.35cm}& & \\ 
    \textsc{ZoomOut+Sync} (bijective \xmark, cycle-consistent \xmark) & \textsc{ZoomOut} (bijective \xmark, cycle-consistent \xmark) & \textsc{ConsistentZoomOut} (bij. \xmark, cycle-cons. \xmark$^\ddagger$)\\
    \includegraphics[width=.08\linewidth]{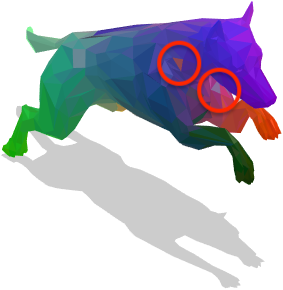} 
    \includegraphics[width=.09\linewidth]{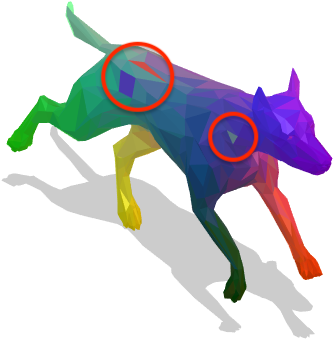} 
    \includegraphics[width=.05\linewidth]{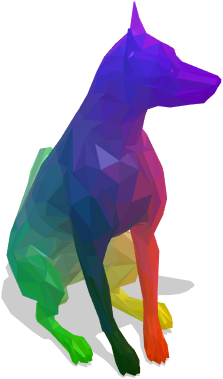} 
    \includegraphics[width=.07\linewidth]{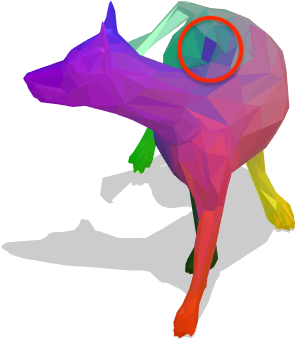} & 
    \includegraphics[width=.08\linewidth]{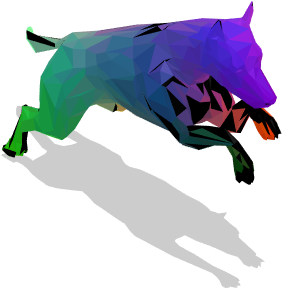} 
    \includegraphics[width=.09\linewidth]{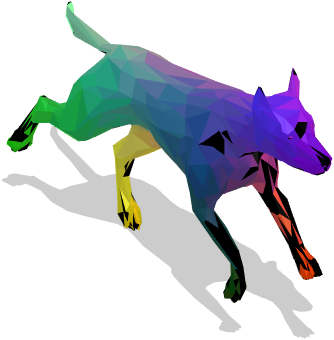} 
    \includegraphics[width=.05\linewidth]{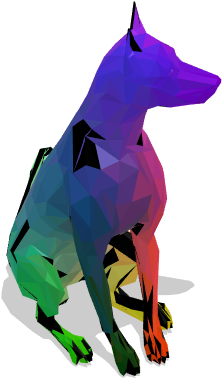} 
    \includegraphics[width=.07\linewidth]{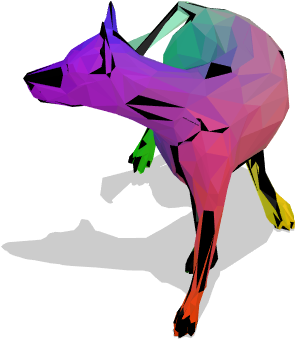} &
    \includegraphics[width=.08\linewidth]{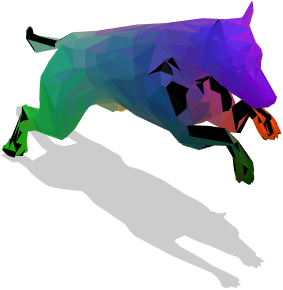} 
    \includegraphics[width=.09\linewidth]{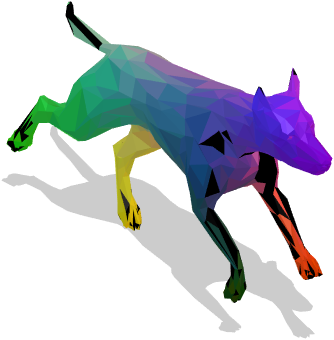} 
    \includegraphics[width=.05\linewidth]{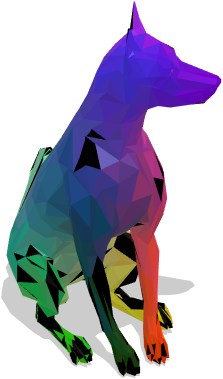} 
    \includegraphics[width=.07\linewidth]{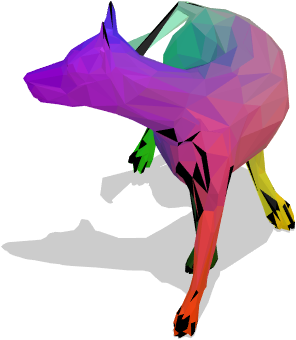}
    \end{tabular}
    \caption{Qualitative examples of correspondences on the TOSCA dog class. Black indicates no matching due to non-bijectivity.
    Our method is cycle-consistent and improves upon the non-smooth and noisy correspondences of the two-stage initialisation obtained via \textsc{ZoomOut+Sync}, whereas \textsc{HiPPI} does not (red circles). \textsc{ZoomOut} and \textsc{ConsistentZoomOut} have many unmatched points (black areas). $^\ddagger$\textsc{ConsistentZoomOut} obtains cycle-consistent $\fm_{ij}$, but not $P_{ij}$. (Best viewed magnified on screen)
    }
    \label{fig:toscaQual} 
\end{figure*}

\paragraph{FAUST dataset.}
The FAUST dataset \cite{bogo2014faust} contains real scans of $10$ different humans in different poses. 
We use the registration subset with $10$ poses for each class and downsample each shape to $2{,}000$ faces.
Our method shows state-of-the-art results on this dataset, see Fig.~\ref{fig:quantitative} and Tab.~\ref{table:resultSummary}.
While the PCK curves between ours, \textsc{ZoomOut+Sync} and \textsc{HiPPI} in Fig.~\ref{fig:quantitative} are  close, the AUC in Tab.~\ref{table:resultSummary} shows that our performance is still superior by a small margin. Qualitative results can be found in the supplementary material.

\begin{figure*}[tbh]
\footnotesize
\centering
\begin{tabular}{c|c|c}
    Colour legend & Ours (bijective \cmark, cycle-consistent \cmark) & \textsc{HiPPI} (bijective \cmark, cycle-consistent \cmark) \\
    \includegraphics[width=.06\linewidth]{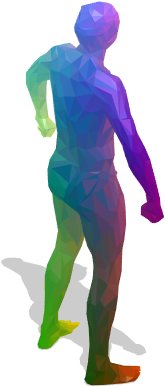} &
    \includegraphics[width=.072\linewidth]{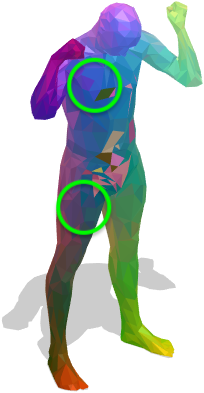} 
    \includegraphics[width=.063\linewidth]{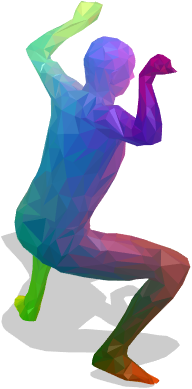} 
    \includegraphics[width=.072\linewidth]{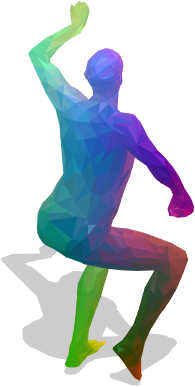} 
    \includegraphics[width=.079\linewidth]{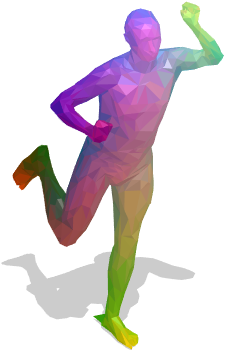} &
    \includegraphics[width=.072\linewidth]{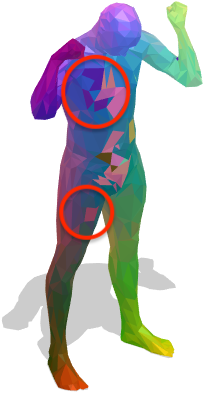} 
    \includegraphics[width=.063\linewidth]{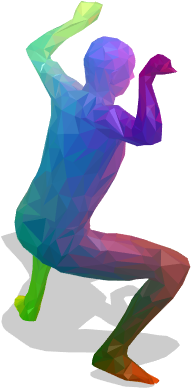} 
    \includegraphics[width=.072\linewidth]{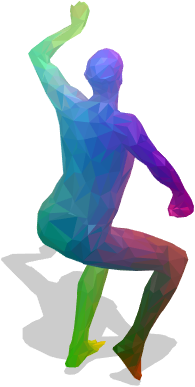} 
    \includegraphics[width=.079\linewidth]{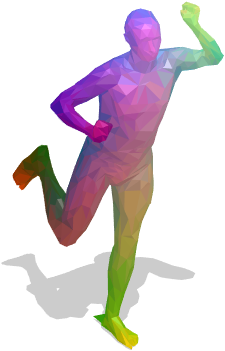} \\
    \hline \vspace{-0.35cm} & & \\ 
    \textsc{ZoomOut+Sync} (bijective \cmark, cycle-consistent \cmark) & \textsc{ZoomOut} (bijective \xmark, cycle-consistent \xmark) & \textsc{ConsistentZoomOut} (bij. \xmark, cycle-cons. \xmark$^\ddagger$) \\
    \includegraphics[width=.072\linewidth]{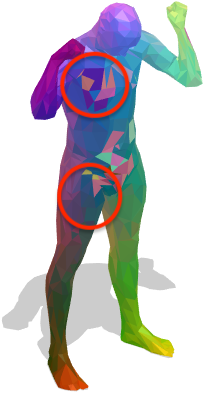} 
    \includegraphics[width=.063\linewidth]{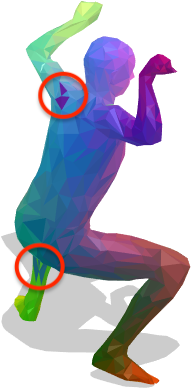} 
    \includegraphics[width=.072\linewidth]{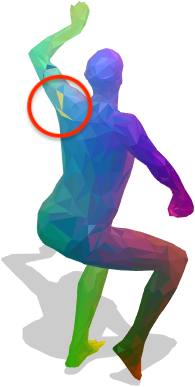} 
    \includegraphics[width=.079\linewidth]{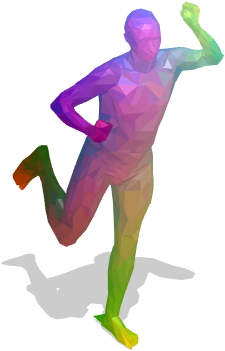} & 
    \includegraphics[width=.072\linewidth]{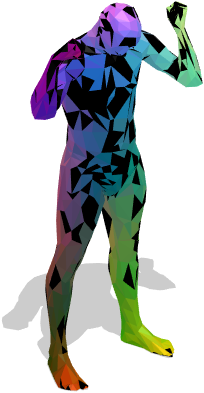} 
    \includegraphics[width=.063\linewidth]{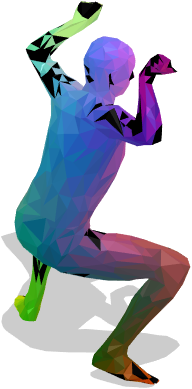} 
    \includegraphics[width=.072\linewidth]{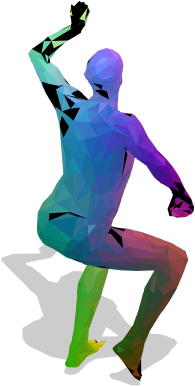} 
    \includegraphics[width=.079\linewidth]{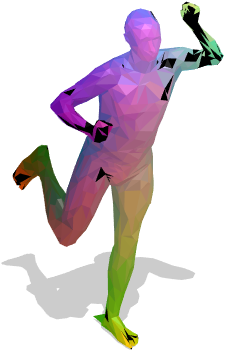} & 
    \includegraphics[width=.072\linewidth]{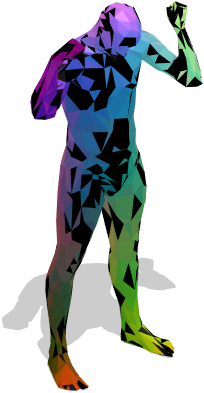} 
    \includegraphics[width=.063\linewidth]{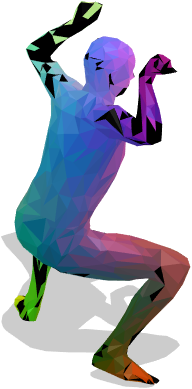} 
    \includegraphics[width=.072\linewidth]{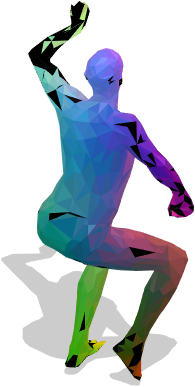} 
    \includegraphics[width=.079\linewidth]{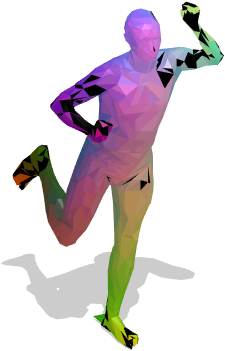}
\end{tabular}
    \caption{Qualitative examples of correspondences on SCAPE. Black indicates no matching due to non-bijectivity. As in Fig.~\ref{fig:toscaQual}, our results contain the least noise and are cycle-consistent, although there is one outlier shape where neither \textsc{HiPPI} nor our method could recover from a bad initialisation. $^\ddagger$\textsc{ConsistentZoomOut} obtains cycle-consistent $\fm_{ij}$, but not $P_{ij}$. (Best viewed magnified on screen)}
    \label{fig:scapeQual} 
\end{figure*}

\paragraph{SCAPE dataset.}
The SCAPE dataset \cite{anguelov2005scape} contains $72$ poses of the same person, of which we chose $15$ randomly and downsample them to $2{,}000$ faces. 
Our method shows state-of-the-art results on this dataset, see Fig.~\ref{fig:quantitative} and Tab.~\ref{table:resultSummary}.
 Exemplary matchings of all methods can be found in Fig.~\ref{fig:scapeQual}.

\subsection{Multi-Matching of Partial Shapes} \label{sec:partial}
We demonstrate that our method applies to the difficult setting of matching partial shapes. As a proof-of-concept, we created a partial dataset by removing several parts of a shape from the TOSCA dataset. 
Most pipelines for partial matching include the full reference shape to resolve some of the complexity. Although our optimisation does not need any information about the complete geometry, we use a partiality-adjusted version of \textsc{ZoomOut} to obtain the shape-to-universe initialisation for \textsc{IsoMuSh}. 
In this case, the optimal universe has the dimension of the full shape. 
Fig.~\ref{fig:teaser} shows that our method finds the correct correspondence among the partial shape collection, while being cycle-consistent. 
Partial functional maps are rectangular and low-rank \cite{rodola2016partial}, and this experiments shows that our method can also handle this more general case. More details can be found in the supplementary material.

\section{Discussion \& Future Work}
\paragraph{Deep learning.} It was shown that deep learning is an extremely powerful approach for extracting 
shape correspondences~\cite{litany2017deep,halimi2019unsupervised,roufosse19unsupervised,groueix19cycleconsistent}. However, the focus of this work is on establishing a fundamental optimisation problem formulation for cycle-consistent isometric multi-shape matching. As such, this work does not focus on learning methods per-se, but we believe that it has a strong potential to spark further work in this direction. In particular, our isometric multi-matching formulation can be integrated into an end-to-end learning framework via differentiable programming techniques~\cite{mena2018learning}. Moreover in machine learning, an entire shape collection is typically used for training, 
so that our multi-matching setting is conceptually better-suited compared to the traditionally used pairwise matching methods. 

\paragraph{Convergence.} We have proven that the \textsc{IsoMuSh} algorithm is convergent in the objective $f(\cdot,\cdot)$. However, we did not establish convergence of the variables $U$ and $Q$. In this context, we note that there are equivalence classes of $U$ and $Q$ that lead to the same objective value. To be more specific, for any (full) $d \times d$ permutation matrix $P$, and any $\fm \in \orth_b$ we have $(UP) \in \perm$, $(Q\fm) \in \orth$, and $f(U,Q) = f(UP,Q\fm)$. The latter can be verified by plugging $UP$ and $Q\fm$ into $f$ while making use of the orthogonality of $P$ and $\fm$. Although the \textsc{IsoMuSh} algorithm is convergent,  and we have empirically verified that it improves upon the state-of-the-art for the isometric multi-shape matching problem, the investigation of stronger convergence results is an interesting direction for future work.

\section{Conclusion}
We presented a novel formulation for the isometric multi-shape matching problem. Our main idea is to simultaneously solve for shape-to-universe matchings and shape-to-universe functional maps. By doing so, we generalise the popular functional map framework to multi-matching, while guaranteeing cycle consistency, both for the shape-to-universe matchings, as well as for the shape-to-universe functional maps. This contrasts the recent \textsc{ConsistentZoomOut}~\cite{huang2020consistent} method, which does not obtain cycle-consistent multi-matchings. Our algorithm is efficient, straightforward to implement, and montonically increases the objective function. Experimentally we have demonstrated that our method outperforms recent state-of-the-art techniques in terms of matching quality, while producing cycle-consistent results and being efficient.

\vspace{0.3cm}
\paragraph{\textbf{Acknowledgements.}}
The authors gracefully acknowledge the support from the ERC Advanced Grant SIMULACRON, the Munich Center for Machine Learning, the CRC "Discretization in Geometry and Dynamics" and the Swedish Research Council (2019-04769).

{\small
\bibliographystyle{ieee_fullname}
\bibliography{paper}
}

\appendix
\clearpage
% \pagebreak
\begin{center}
\textbf{\large Supplementary Material}
\end{center}

\section{Theoretical Analysis (with Proofs)}
\begin{lem}\label{lem:U}
$\langle U_{t}^\top \mathbf{\Phi} Q_t , U_{t+1}^\top \mathbf{\Phi} Q_t \rangle \geq \langle U_t^\top \mathbf{\Phi} Q_t , U_t^\top \mathbf{\Phi} Q_t \rangle$ holds for any $t$.
\end{lem}
\begin{proof}
According to \eqref{eq:projP},  the function $\langle U_t^\top \mathbf{\Phi} Q_t, U \mathbf{\Phi} Q_t \rangle$ is maximised w.r.t. $U$ over $\perm$ for the choice $U = U_{t+1}$. 
Our claim follows immediately from this.
\end{proof}

\begin{prop}[Monotonicity of $U$-update]\label{prop:U}\ \\
The objective values cannot decrease through the $U$-update step~\eqref{eq:projP}, and $\langle U_{t+1}^\top \mathbf{\Phi} Q_t , U_{t+1}^\top \mathbf{\Phi} Q_t \rangle \geq \langle U_t^\top \mathbf{\Phi} Q_t , U_t^\top \mathbf{\Phi} Q_t \rangle$ holds.
\end{prop}
\begin{proof}
We prove the proposition by using Lemma~\ref{lem:U}. Recalling that $Z= \mathbf{\Phi} Q_t Q_t^\top \mathbf{\Phi}^\top$, we can see that
     \begin{align}
             0 &\leq \|U_{t+1}^\top \mathbf{\Phi} Q_t - U_t^\top \mathbf{\Phi} Q_t \|_F^2 \\
            & = \langle U_{t+1}^\top \mathbf{\Phi} Q_t, U_{t+1}^\top \mathbf{\Phi} Q_t \rangle - 2 \langle U_{t+1}^\top \mathbf{\Phi} Q_t,U_t^\top \mathbf{\Phi} Q_t \rangle \nonumber \\
            & \quad + \langle U_t^\top \mathbf{\Phi} Q_t, U_t^\top \mathbf{\Phi} Q_t \rangle  \\
            &= \langle U_{t+1}^\top, U_{t+1}^\top Z \rangle - 2 \langle U_{t+1}^\top,U_t^\top Z\rangle + \langle U_t^\top, U_t Z \rangle  \,.
     \end{align}
     From Lemma~\ref{lem:U} and using the symmetry of Z, we know that $\langle U_{t+1}^\top, U_t^\top Z \rangle \geq \langle U_t^\top, U_{t}^\top Z \rangle$. 
     By transitivity this leads to
     \begin{align*}
         &0 \leq \langle U_{t+1}^\top, U_{t+1}^\top Z \rangle - 2 \langle U_{t+1}^\top,U_t^\top Z\rangle + \langle U_{t+1}^\top, U_t^\top Z \rangle\,,
         \intertext{so that}
          &\  \langle U_{t+1}^\top, U_t^\top Z \rangle \leq \langle U_{t+1}^\top, U_{t+1}^\top Z \rangle\,.
     \end{align*}
\end{proof}

\begin{lem}\label{lem:Q} In each iteration $t$,
$\langle U_{t+1}^\top \mathbf{\Phi} Q_t , U_{t+1}^\top \mathbf{\Phi} Q_{t+1} \rangle \geq \langle U_{t+1}^\top \mathbf{\Phi} Q_t , U_{t+1}^\top \mathbf{\Phi} Q_t \rangle$ holds.
\end{lem}
\begin{proof}
Analogously to the proof of Lemma~\ref{lem:U}, and according to \eqref{eq:projO}, the choice $Q = Q_{t+1}$ is the element maximising the expression $\langle U_{t+1}^\top \mathbf{\Phi} Q_t , U_{t+1}^\top \mathbf{\Phi} Q \rangle$ w.r.t. $Q$ over $\orth$.
\end{proof}

\begin{prop}[Monotonicity of $Q$-update]\label{prop:Q} \ \\
The objective values cannot decrease through the $Q$-update~\eqref{eq:projO}, and $\langle U_{t+1}^\top \mathbf{\Phi} Q_{t+1} , U_{t+1}^\top \mathbf{\Phi} Q_{t+1} \rangle \geq \langle U_{t+1}^\top \mathbf{\Phi} Q_t , U_{t+1}^\top \mathbf{\Phi} Q_t \rangle$ holds.
\end{prop}
\begin{proof}
The proof is analogous to Prop.~\ref{prop:U}. For $\overline{Z} = \mathbf{\Phi}^\top U_{t+1} U_{t+1}^\top \mathbf{\Phi}$ we observe that
\begin{align}
    0 &\leq \|  U_{t+1}^\top \mathbf{\Phi} Q_t - U_{t+1}^\top \mathbf{\Phi} Q_{t+1} \|^2_F \\
    & = \langle Q_t, \overline{Z}Q_t \rangle  - 2\langle Q_t, \overline{Z}Q_{t+1} \rangle  + \langle Q_{t+1}, \overline{Z}Q_{t+1} \rangle \,.
\end{align}
From Lemma~\ref{lem:Q} we have $\langle Q_{t}^\top, \overline{Z}Q_{t+1} \rangle \geq \langle Q_{t}^\top, \overline{Z}Q_{t} \rangle$, so that our claim follows by transitivity.
\end{proof}

By combining these properties we obtain the following immediate result regarding Algorithm~\ref{alg:isomush}:
\begin{theorem}[Convergence]\ \\
The sequence $(f(U_t,Q_t))_{t=1,2,\ldots}$ is convergent and  Algorithm~\ref{alg:isomush} terminates in finite time.
\end{theorem}
\begin{proof}
For any $t$ we have  $U_t \in \perm$ and $Q_t \in \orth$. Hence, the value of  $f(U_t,Q_t)$ is bounded from above (both $\perm$ and $\orth$ are compact sets). Combined with the monotonicity of the $U$-update (Prop.~\ref{prop:U}) and $Q$-update (Prop.~\ref{prop:Q}), this shows that the sequence $(f(U_t,Q_t))_{t=1,2\ldots}$ converges.
\end{proof}

\begin{figure*}[tbh]
\centering
    \begin{tabular}{cc}
        \includegraphics[width=.14\linewidth]{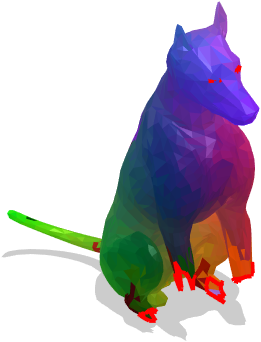} & \includegraphics[width=.14\linewidth]{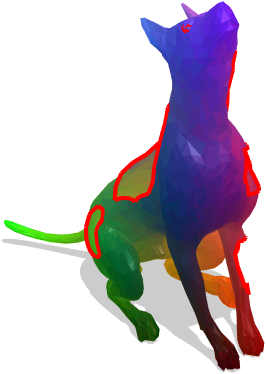} \\
        \includegraphics[width=.14\linewidth]{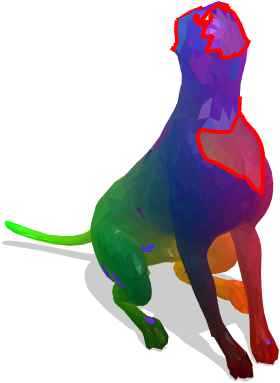} & \includegraphics[width=.11\linewidth]{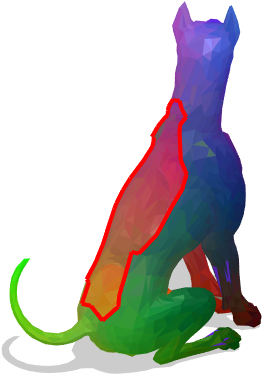}
    \end{tabular}
    \begin{tabular}{cc}
        \includegraphics[width=.17\linewidth]{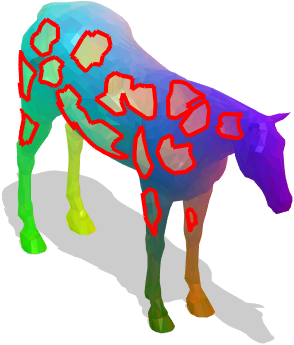} & \includegraphics[width=.17\linewidth]{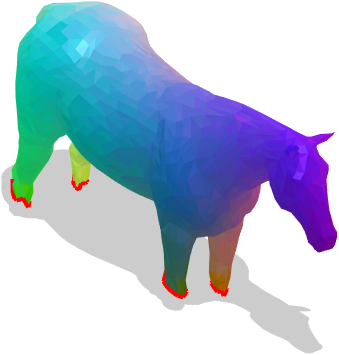} \\
        \includegraphics[width=.17\linewidth]{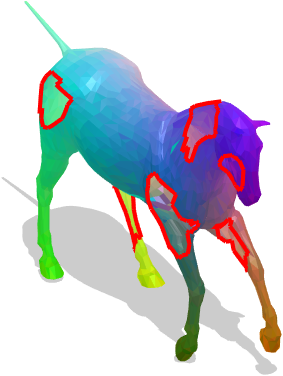} & \includegraphics[width=.17\linewidth]{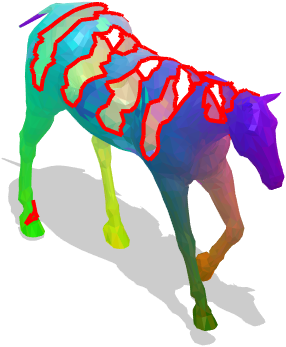}
    \end{tabular}
    \caption{Our results on deformed partial shapes of the TOSCA dog and horse classes. Although partial-to-partial matching is a very challenging setting, our method produces high quality results. These results serve as a proof-of-concept that our method is applicable to such partial settings.}
    \label{fig:partialQual}
\end{figure*}

\section{Details on Experimental Setup}
In this section we clarify further details regarding the experimental setup.

\paragraph{Synchronisation of \textsc{ZoomOut} results.}
Running \textsc{ZoomOut} produces the pairwise correspondences $\{P_{ij}\}$,  and the pairwise functional maps $\{\fm_{ij}\}$  between all pairs of shapes. In general, the pairwise correspondences and the pairwise functional maps are not cycle-consistent. In order to obtain cycle-consistent shape-to-universe representations we apply synchronisation, as we explain next.

For isometric shapes, the spectra of the Laplace-Beltrami operator are the same for all shapes. Moreover, the pairwise functional maps have a band-diagonal structure, where the band-width depends on the largest multiplicity of the spectra, see~\cite{maron2016point} for details. Hence, we first set all elements of $\fm_{ij}$ to $0$ that are outside the diagonal band of radius $r=6$,~i.e.~$\fm_{ij}(s,t) = 0$ whenever $|s-t| > r$. Subsequently, we project the ``band-filtered'' $\fm_{ij}$ onto $\orth_b$ using singular value decomposition (since isometric shapes must lead to orthogonal functional maps). Eventually, we use orthogonal transformation synchronisation~\cite{Singer:2011ba, Bernard_2015_CVPR} in order to obtain the shape-to-universe functional maps $\{\fm_i\}$, which we stack into the block-matrix $Q$.

In order to obtain $U$, we first represent all LBO eigenfunctions $\Phi_i$ in terms of the universe,~i.e.~$\Phi_i\fm_i$, and then stack them all into the matrix
\begin{align}
    \mathbf{\Psi} = 
    \begin{bmatrix}
    \Phi_1\fm_1 \\ \vdots \\ \Phi_k\fm_k
    \end{bmatrix}\,.
\end{align}
Eventually, we obtain the shape-to-universe matching matrix $U \in \perm$ by performing a constrained clustering, where the features used for clustering are the inner products between the eigenfunctions in the universe representation. This means that the rows of the matrix $\mathbf{\Psi}\mathbf{\Psi}^\top$ are used as features for clustering. This is motivated by the (constrained) clustering interpretation of partial permutation synchronisation, see~e.g.~\cite{Tron:kUBrCZhd,bernard2019synchronisation}. For performing the clustering, we first apply the Successive Block Rotation Algorithm (SBRA)~\cite{bernard2019synchronisation}, followed by projecting the result onto the set $\perm$. Further details can be found in~\cite{bernard2019synchronisation}.

\paragraph{Symmetries.}
Bringing symmetric shapes into correspondence is well-known to be a challenging problem~\cite{sun2018joint}. To avoid symmetric flips it is common practice to incorporate an additional symmetry descriptor into shape matching formulations, as for example done in~\cite{cosmo2017consistent}. We follow this path, and make use of a symmetry descriptor for finding the \textsc{ZoomOut} initialisation. We emphasise that the symmetry descriptor is not used after the multi-shape matching methods have been initialised.

\paragraph{Parameters.}
For the experiments that consider full shapes (on the TOSCA, FAUST and SCAPE datasets), there exists a bijection between all shapes within a category, hence $m_i=m_j$ for all $i,j$. Thus, we set the universe size $d$ to the number of vertices present in each shape,~i.e~$d=m_i$. In all experiments, we fix the relative objective improvement to machine precision,~i.e.~$\epsilon \approx 2.2{\cdot}10^{-16}$. %

\section{Multi-Matching of Partial Shapes}

This section will provide more details on the experiments of Section~\ref{sec:partial} in the main paper. 
Strictly speaking, partial shapes do not fulfil the isometry assumption due to missing parts that affect geodesic distances. However, in the case of finding a matching between a full shape and shape with holes, both of the same class, there is a close relationship
(see Fig.~\ref{fig:teaser}).
\cite{rodola2016partial} discusses how spectral properties change in this case, and the necessary adjustment of our pipeline is based on this theory.
Finding correspondences for partial-to-partial cases is a much more challenging and open problem, and due to a lack of robust initialisations, as a proof-of-concept we show results on small datasets with only minor deformations. 
See Fig.~\ref{fig:partialQual} for qualitative results.%

\paragraph{Problem formulation.}
Partiality can be handled naturally in our approach due to the universe formulation. 
Since each $P_i$ maps the points of $\mathcal{X}_i$ to a \textit{subset} of the $d$ universe points, this case boils down to choosing the correct universe points. 
Assuming that all given partial shapes represent parts of the same full shape, the optimal universe would model exactly the full geometry. 

The functional maps $\fm_i$ need to be adjusted slightly for this setting. 
As explained in Section~\ref{sec:fms} in the main paper, square orthogonal $\fm$s model area-preservation. 
This is meaningful for isometries, but, since partial shapes literally miss some areas, it does not hold in this case. 
Instead, we use the theory about partial functional maps provided in \cite{rodola2016partial}.
According to \cite{rodola2016partial}, functional maps for the partial case have \textit{slanted} diagonals and the area preservation only holds in one direction. Additionally, some LBO eigenfunctions of the full shape do not appear on the partial shapes, such that each $\fm_i$ needs to map to a higher dimensional space, and only choose the corresponding eigenfunctions there. Therefore, instead of being square, the matrices are rectangular, and we adjust the definition of the orthogonality constraint as
\begin{align}\label{eq:rectorth}
    \orth_b^P = \left\{ \fm \in \mathbb{R}^{b\times b'}: \fm \fm^\top = \mathbf{I}_b \right\} \,,
\end{align} 
where $b' > b$, and we chose $b' = 1.2b$ in all our partial experiments. Note that it does not require any modification in our optimisation pipeline and our problem formulation is capable of handling this more challenging case.

\paragraph{Initialisation.} 
For the full multi-shape matching pipeline, we used functional maps \cite{ovsjanikov2012functional} and \textsc{ZoomOut} \cite{melzi19zoomout} to get an estimation for each $\fm_{ij}$. 
However, they are not well-suited for directly performing partial-to-partial matching.
Instead, we directly compute $\{P_i\}$ between each partial and the full shape using a combination of SHOT \cite{salti2014shot}, Heat Kernel Signature \cite{bronstein2010scale}, Wave Kernel Signature \cite{aubry2011wave} and symmetry descriptors, which are subsequently refined using a partiality-adjusted version of \textsc{ZoomOut} to obtain the shape-to-universe initialisation for \textsc{IsoMuSh}.

\section{Additional Qualitative Results}
We show qualitative results on FAUST in Fig.~\ref{fig:faustQual}, the complete results on SCAPE in Fig.~\ref{fig:scapeAlls}, as well as additional qualitative results of different TOSCA classes in Fig.~\ref{fig:toscaHorse} and Fig.~\ref{fig:toscaMichael}.
Fig.~\ref{fig:faustQual} shows the main source of errors for our method of FAUST, which are front-back flips. 
This is due to the intrinsic front-back near-symmetry of humans and descriptors that do not discriminate well between these. 
All \textsc{ZoomOut} variants (including the initialisation for our method) suffer from this problem.
Note that even though the correspondence is flipped, our results are still cycle-consistent.

\begin{figure*}[tbh]
\footnotesize
\centering
\begin{tabular}{c|c}
    Colour legend & Ours (bijective \cmark, cycle-consistent \cmark) \\
    \includegraphics[width=.115\linewidth]{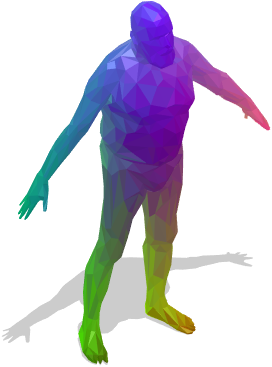} &
    \includegraphics[width=.085\linewidth]{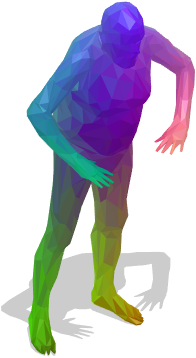} 
    \includegraphics[width=.075\linewidth]{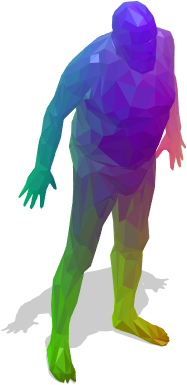} 
    \includegraphics[width=.065\linewidth]{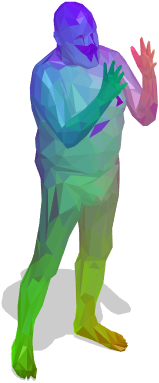} 
    \includegraphics[width=.09\linewidth]{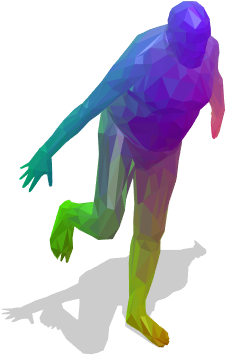} 
    \includegraphics[width=.08\linewidth]{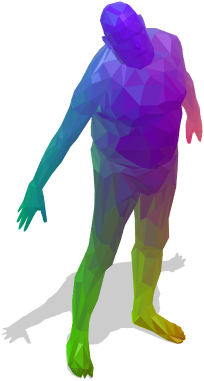} 
    \includegraphics[width=.075\linewidth]{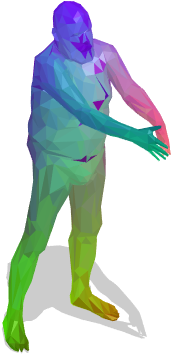} 
    \includegraphics[width=.075\linewidth]{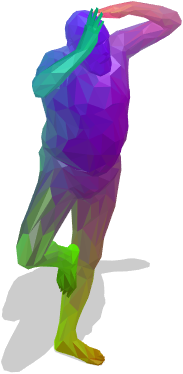} 
    \includegraphics[width=.095\linewidth]{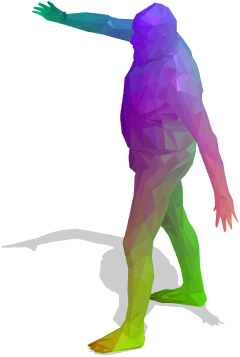} 
    \includegraphics[width=.085\linewidth]{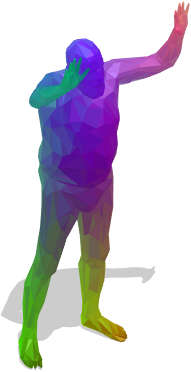}  \\
    \hline \vspace{-0.35cm} &  \\ 
     & \textsc{HiPPI} (bijective \cmark, cycle-consistent \cmark) \\
    & \includegraphics[width=.085\linewidth]{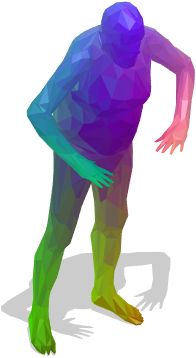} 
    \includegraphics[width=.075\linewidth]{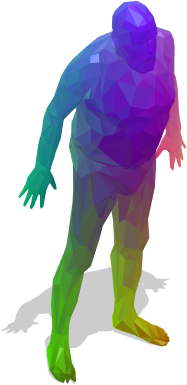} 
    \includegraphics[width=.065\linewidth]{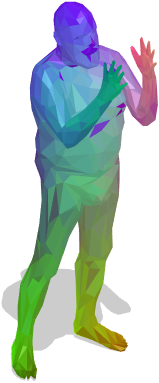} 
    \includegraphics[width=.09\linewidth]{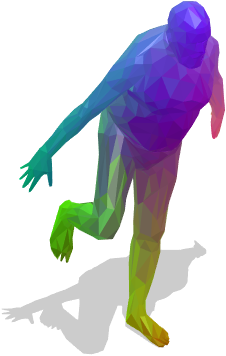} 
    \includegraphics[width=.08\linewidth]{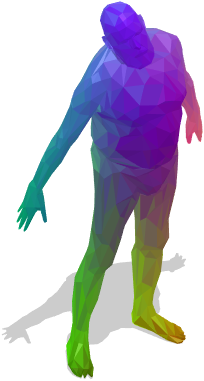} 
    \includegraphics[width=.075\linewidth]{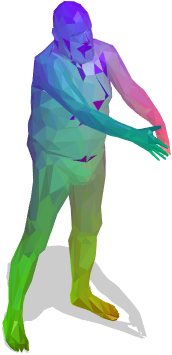} 
    \includegraphics[width=.075\linewidth]{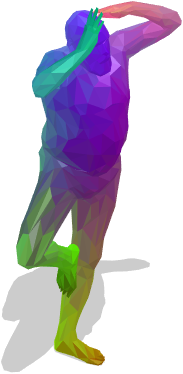} 
    \includegraphics[width=.095\linewidth]{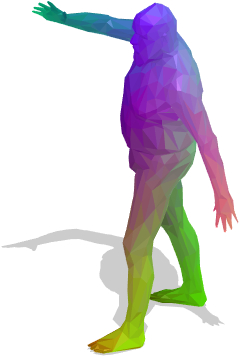} 
    \includegraphics[width=.085\linewidth]{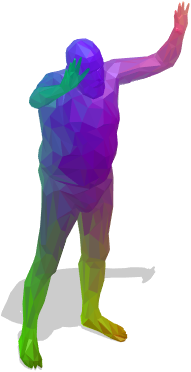}  \\
    \hline \vspace{-0.35cm} &  \\ 
    & \textsc{ZoomOut+Sync} (bijective \cmark, cycle-consistent \cmark) \\
    & \includegraphics[width=.085\linewidth]{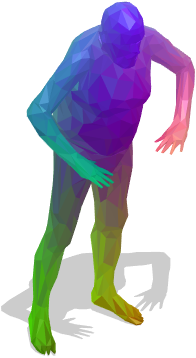} 
    \includegraphics[width=.075\linewidth]{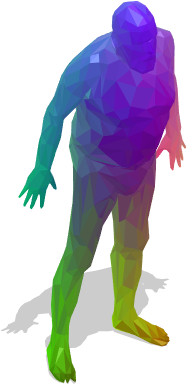} 
    \includegraphics[width=.065\linewidth]{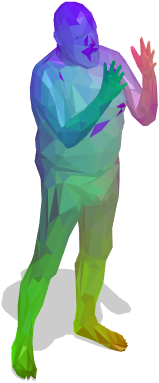} 
    \includegraphics[width=.09\linewidth]{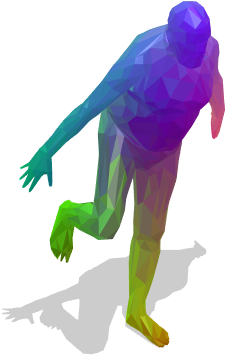} 
    \includegraphics[width=.08\linewidth]{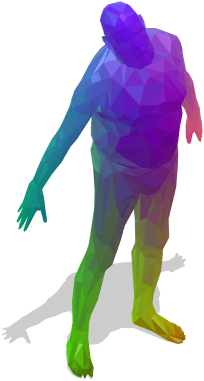} 
    \includegraphics[width=.075\linewidth]{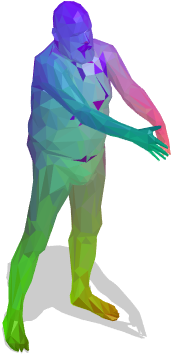} 
    \includegraphics[width=.075\linewidth]{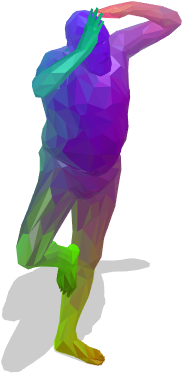} 
    \includegraphics[width=.095\linewidth]{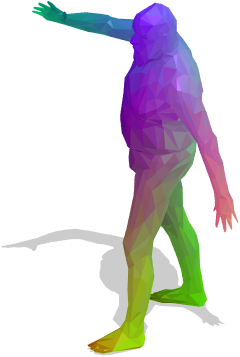} 
    \includegraphics[width=.085\linewidth]{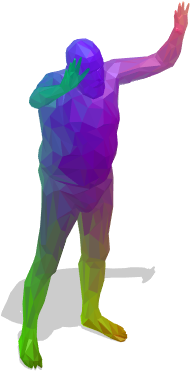}  \\
    \hline \vspace{-0.35cm} &  \\ 
    & \textsc{ZoomOut} (bijective \xmark, cycle-consistent \xmark) \\
    & \includegraphics[width=.085\linewidth]{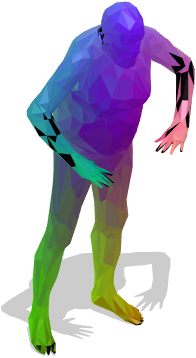} 
    \includegraphics[width=.075\linewidth]{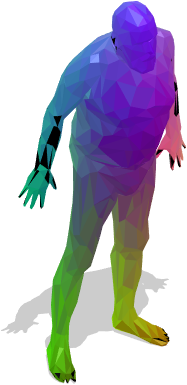} 
    \includegraphics[width=.065\linewidth]{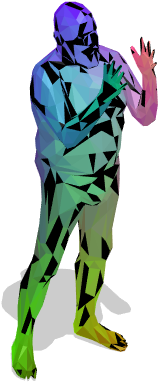} 
    \includegraphics[width=.09\linewidth]{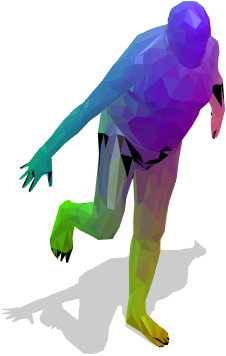} 
    \includegraphics[width=.08\linewidth]{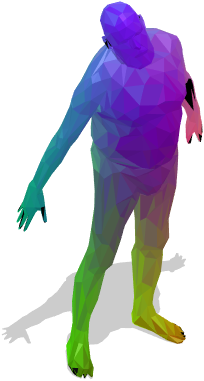} 
    \includegraphics[width=.075\linewidth]{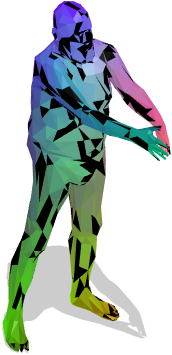} 
    \includegraphics[width=.075\linewidth]{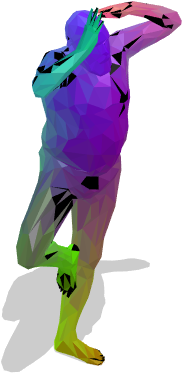} 
    \includegraphics[width=.095\linewidth]{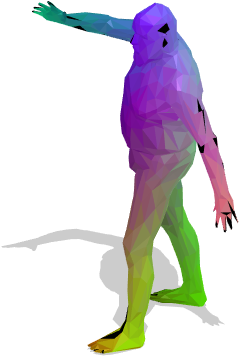} 
    \includegraphics[width=.085\linewidth]{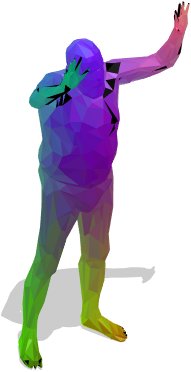} \\
    \hline \vspace{-0.35cm} &  \\ 
    & \textsc{ConsistentZoomOut} (bij. \xmark, cycle-cons. \xmark$^\ddagger$) \\
    & \includegraphics[width=.085\linewidth]{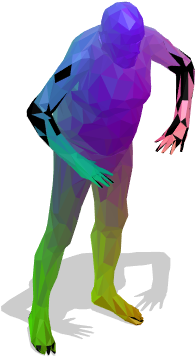} 
    \includegraphics[width=.075\linewidth]{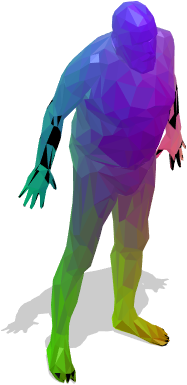} 
    \includegraphics[width=.065\linewidth]{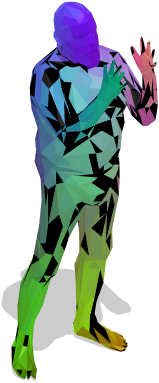} 
    \includegraphics[width=.09\linewidth]{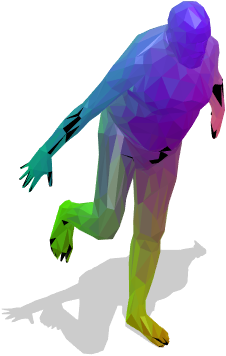} 
    \includegraphics[width=.08\linewidth]{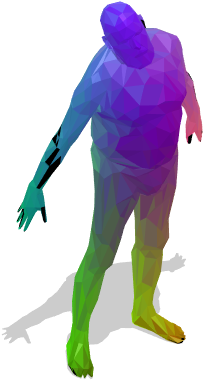} 
    \includegraphics[width=.075\linewidth]{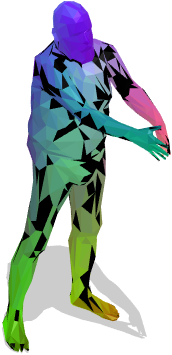} 
    \includegraphics[width=.075\linewidth]{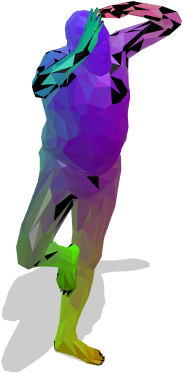} 
    \includegraphics[width=.095\linewidth]{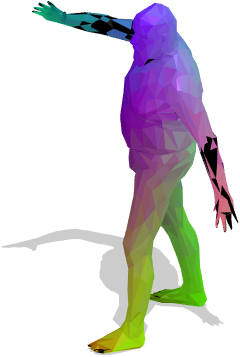} 
    \includegraphics[width=.085\linewidth]{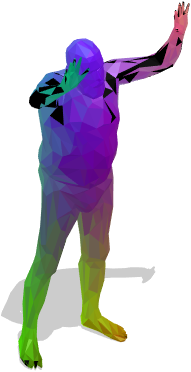}
\end{tabular}
    \caption{Qualitative examples of correspondences on FAUST registrations. Black indicates no matching due to non-bijectivity. Our results contain the least noise and are cycle-consistent. $^\ddagger$\textsc{ConsistentZoomOut} obtains cycle-consistent $\fm_{ij}$, but not $P_{ij}$. (Best viewed magnified on screen)}
    \label{fig:faustQual} 
\end{figure*}

\begin{figure*}[tbh]
\footnotesize
\centering
\begin{tabular}{c|c}
    Colour legend & Ours (bijective \cmark, cycle-consistent \cmark) \\
    \includegraphics[width=.05\linewidth]{figures/qual/zoomout-syncPfromPhi-sbra-ini-scape15-det=0_mesh000_mesh004_source2.png} &
    \includegraphics[width=.06\linewidth]{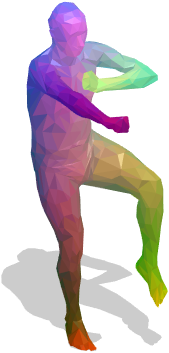}
    \includegraphics[width=.055\linewidth]{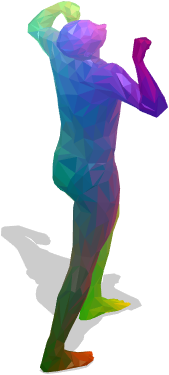}
    \includegraphics[width=.06\linewidth]{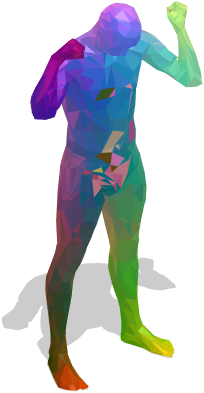}
    \includegraphics[width=.065\linewidth]{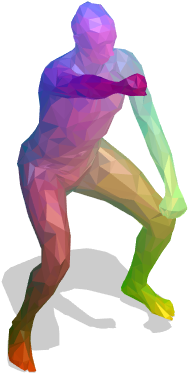}
    \includegraphics[width=.06\linewidth]{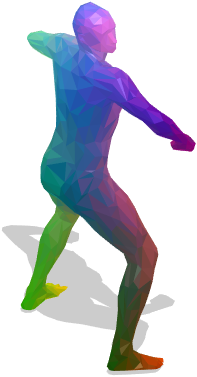}
    \includegraphics[width=.06\linewidth]{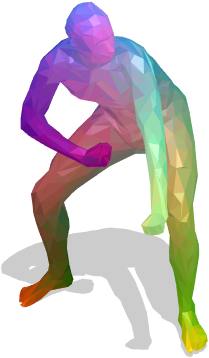}
    \includegraphics[width=.07\linewidth]{figures/qual/zoomout-syncPfromPhi-sbra-scape15-det=0_mesh000_mesh042_target2.png}
    \includegraphics[width=.05\linewidth]{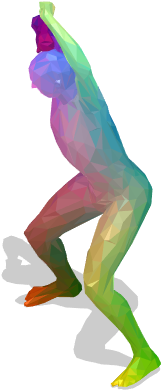}
    \includegraphics[width=.05\linewidth]{figures/qual/zoomout-syncPfromPhi-sbra-scape15-det=0_mesh000_mesh052_target2.png}
    \includegraphics[width=.05\linewidth]{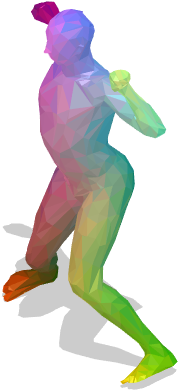}
    \includegraphics[width=.07\linewidth]{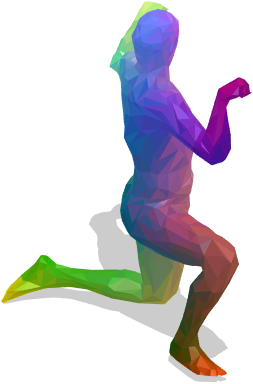}
    \includegraphics[width=.065\linewidth]{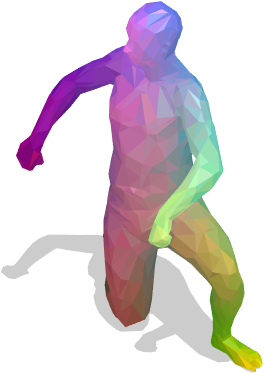}
    \includegraphics[width=.06\linewidth]{figures/qual/zoomout-syncPfromPhi-sbra-scape15-det=0_mesh000_mesh061_target2.png}
    \includegraphics[width=.05\linewidth]{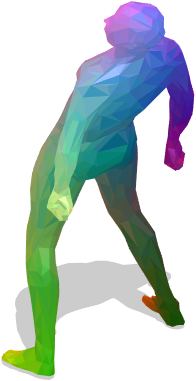}
     \\
    \hline \vspace{-0.35cm} &  \\ 
     & \textsc{HiPPI} (bijective \cmark, cycle-consistent \cmark) \\
    & 
    \includegraphics[width=.06\linewidth]{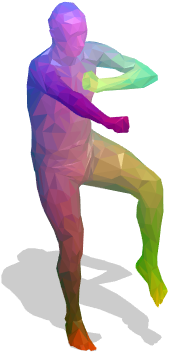}
    \includegraphics[width=.055\linewidth]{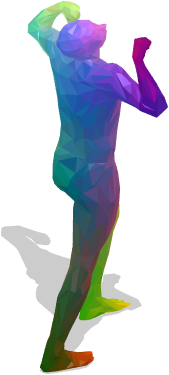}
    \includegraphics[width=.06\linewidth]{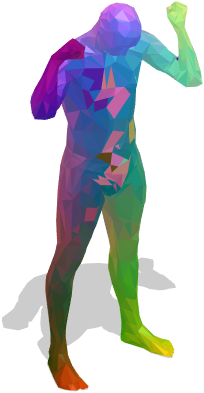}
    \includegraphics[width=.065\linewidth]{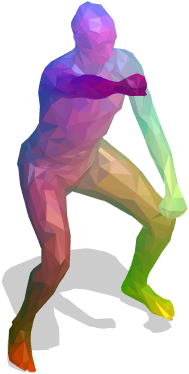}
    \includegraphics[width=.06\linewidth]{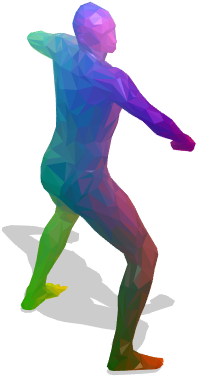}
    \includegraphics[width=.06\linewidth]{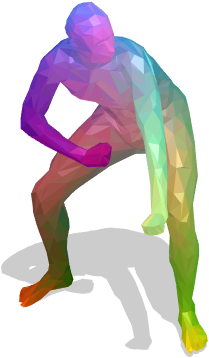}
    \includegraphics[width=.07\linewidth]{figures/qual/zoomout-syncPfromPhi-sbra-hippi-orig-scape15_mesh000_mesh042_target2.png}
    \includegraphics[width=.05\linewidth]{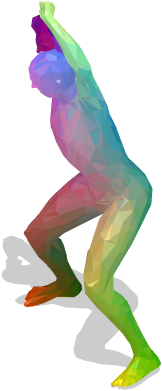}
    \includegraphics[width=.05\linewidth]{figures/qual/zoomout-syncPfromPhi-sbra-hippi-orig-scape15_mesh000_mesh052_target2.png}
    \includegraphics[width=.05\linewidth]{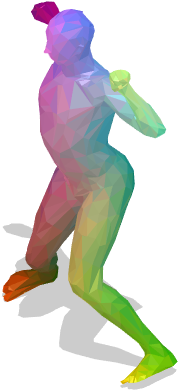}
    \includegraphics[width=.07\linewidth]{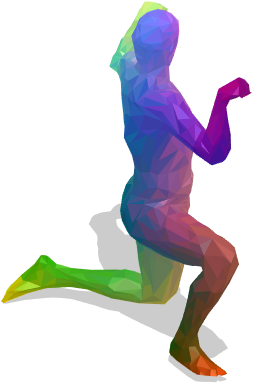}
    \includegraphics[width=.065\linewidth]{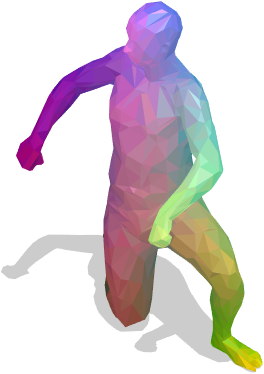}
    \includegraphics[width=.06\linewidth]{figures/qual/zoomout-syncPfromPhi-sbra-hippi-orig-scape15_mesh000_mesh061_target2.png}
    \includegraphics[width=.05\linewidth]{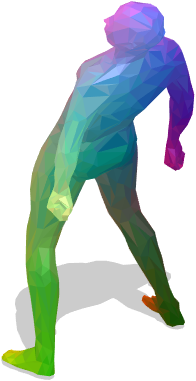}
    \\
    \hline \vspace{-0.35cm} &  \\ 
    & \textsc{ZoomOut+Sync} (bijective \cmark, cycle-consistent \cmark) \\
    & 
    \includegraphics[width=.06\linewidth]{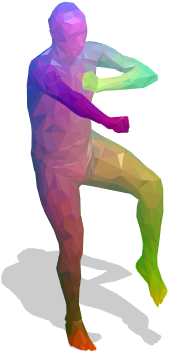}
    \includegraphics[width=.055\linewidth]{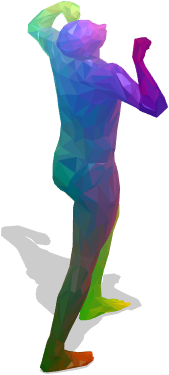}
    \includegraphics[width=.06\linewidth]{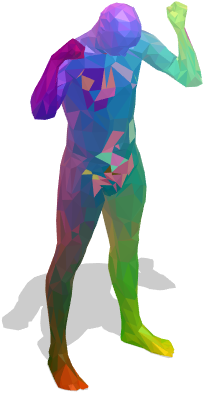}
    \includegraphics[width=.065\linewidth]{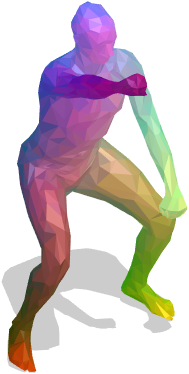}
    \includegraphics[width=.06\linewidth]{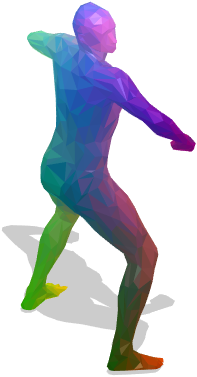}
    \includegraphics[width=.06\linewidth]{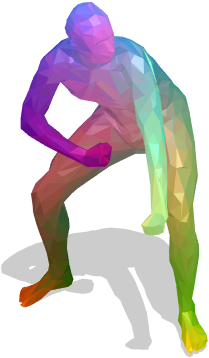}
    \includegraphics[width=.07\linewidth]{figures/qual/zoomout-syncPfromPhi-sbra-ini-scape15-det=0_mesh000_mesh042_target2.png}
    \includegraphics[width=.05\linewidth]{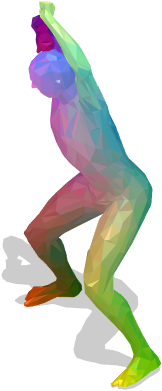}
    \includegraphics[width=.05\linewidth]{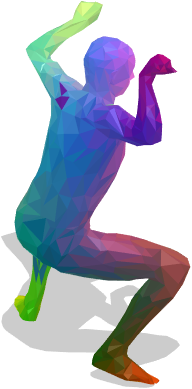}
    \includegraphics[width=.05\linewidth]{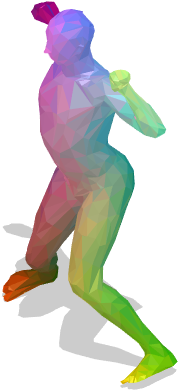}
    \includegraphics[width=.07\linewidth]{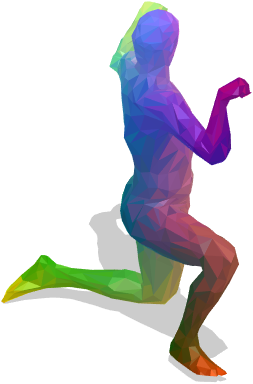}
    \includegraphics[width=.065\linewidth]{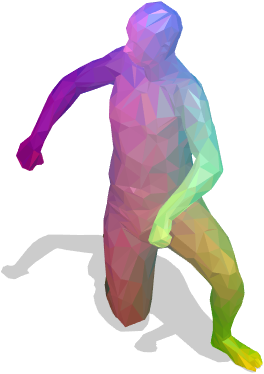}
    \includegraphics[width=.06\linewidth]{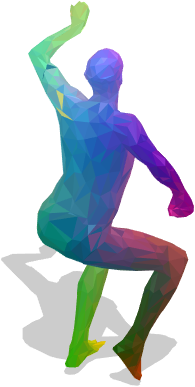}
    \includegraphics[width=.05\linewidth]{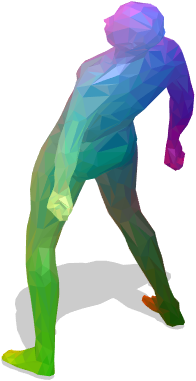}
    \\
    \hline \vspace{-0.35cm} &  \\ 
    & \textsc{ZoomOut} (bijective \xmark, cycle-consistent \xmark) \\
    & 
    \includegraphics[width=.06\linewidth]{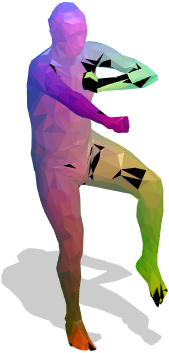}
    \includegraphics[width=.055\linewidth]{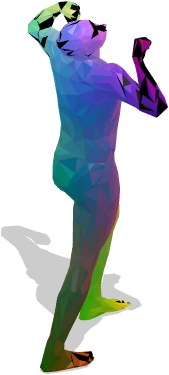}
    \includegraphics[width=.06\linewidth]{figures/qual/zoomout-orig-scape15_mesh000_mesh021_target2.png}
    \includegraphics[width=.065\linewidth]{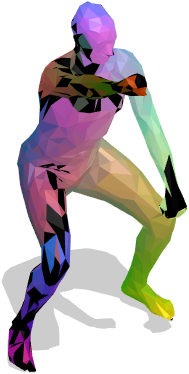}
    \includegraphics[width=.06\linewidth]{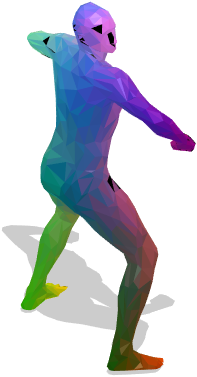}
    \includegraphics[width=.06\linewidth]{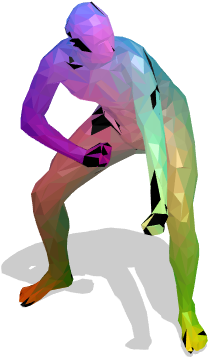}
    \includegraphics[width=.07\linewidth]{figures/qual/zoomout-orig-scape15_mesh000_mesh042_target2.png}
    \includegraphics[width=.05\linewidth]{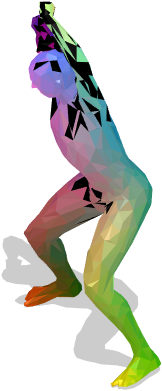}
    \includegraphics[width=.05\linewidth]{figures/qual/zoomout-orig-scape15_mesh000_mesh052_target2.png}
    \includegraphics[width=.05\linewidth]{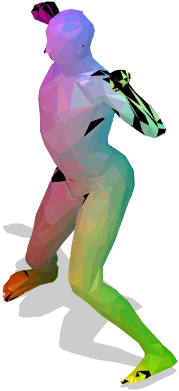}
    \includegraphics[width=.07\linewidth]{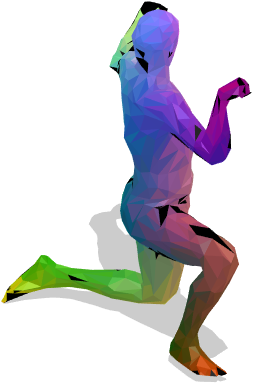}
    \includegraphics[width=.065\linewidth]{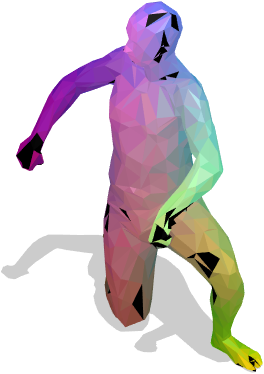}
    \includegraphics[width=.06\linewidth]{figures/qual/zoomout-orig-scape15_mesh000_mesh061_target2.png}
    \includegraphics[width=.05\linewidth]{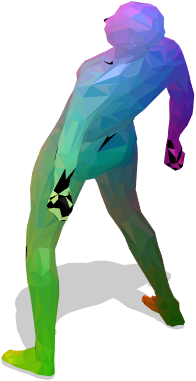}
    \\
    \hline \vspace{-0.35cm} &  \\ 
    & \textsc{ConsistentZoomOut} (bij. \xmark, cycle-cons. \xmark$^\ddagger$) \\
    & 
    \includegraphics[width=.06\linewidth]{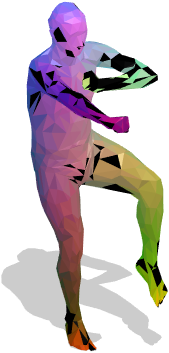}
    \includegraphics[width=.055\linewidth]{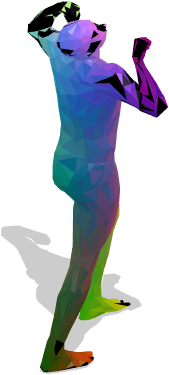}
    \includegraphics[width=.06\linewidth]{figures/qual/consistent_zoomout_scape15_lb100_nf2000_mesh000_mesh021_target2.png}
    \includegraphics[width=.065\linewidth]{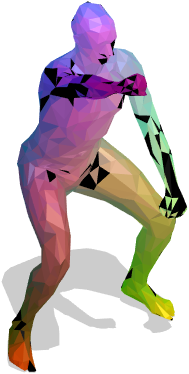}
    \includegraphics[width=.06\linewidth]{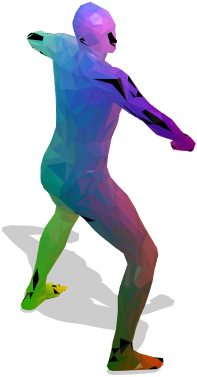}
    \includegraphics[width=.06\linewidth]{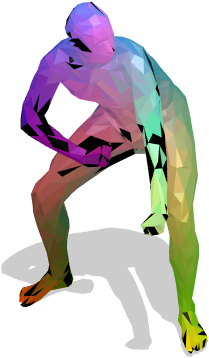}
    \includegraphics[width=.07\linewidth]{figures/qual/consistent_zoomout_scape15_lb100_nf2000_mesh000_mesh042_target2.png}
    \includegraphics[width=.05\linewidth]{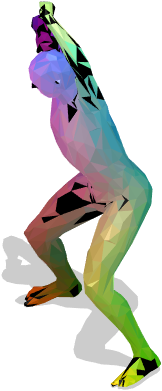}
    \includegraphics[width=.05\linewidth]{figures/qual/consistent_zoomout_scape15_lb100_nf2000_mesh000_mesh052_target2.png}
    \includegraphics[width=.05\linewidth]{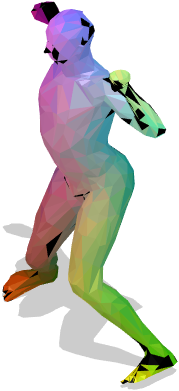}
    \includegraphics[width=.07\linewidth]{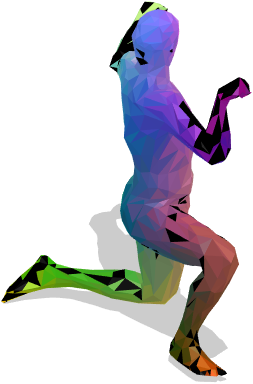}
    \includegraphics[width=.065\linewidth]{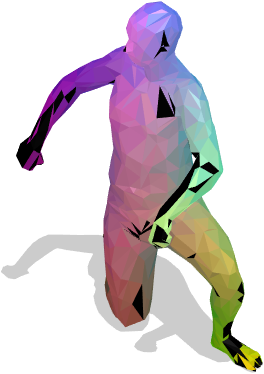}
    \includegraphics[width=.06\linewidth]{figures/qual/consistent_zoomout_scape15_lb100_nf2000_mesh000_mesh061_target2.png}
    \includegraphics[width=.05\linewidth]{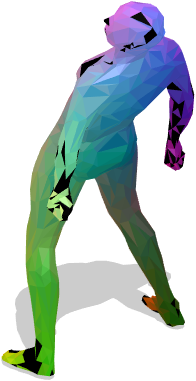}
\end{tabular}
    \caption{Complete qualitative results of correspondences on SCAPE. Black indicates no matching due to non-bijectivity. Our results contain the least noise and are cycle-consistent, although there is one outlier shape where neither \textsc{HiPPI} nor our method could recover from a bad initialisation. $^\ddagger$\textsc{ConsistentZoomOut} obtains cycle-consistent $\fm_{ij}$, but not $P_{ij}$. (Best viewed magnified on screen)}
    \label{fig:scapeAlls} 
\end{figure*}

\begin{figure*}[tbh]
\footnotesize
\centering
\begin{tabular}{c|c}
    Colour legend & Ours (bijective \cmark, cycle-consistent \cmark) \\
    \includegraphics[width=.115\linewidth]{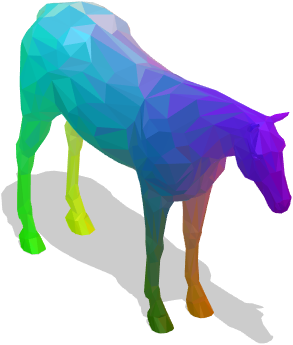} &
    \includegraphics[width=.11\linewidth]{figures/qual/zoomout-syncPfromPhi-sbra_horse0_horse5_target.png}
    \includegraphics[width=.11\linewidth]{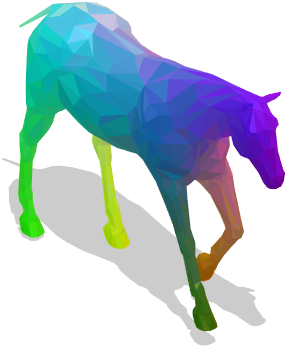}
    \includegraphics[width=.13\linewidth]{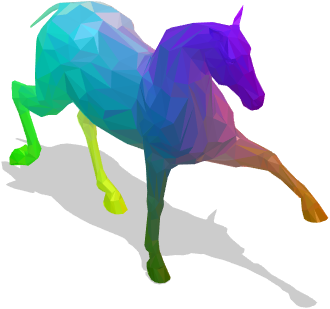}
    \includegraphics[width=.11\linewidth]{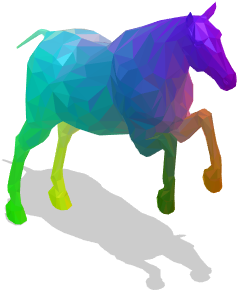}
    \includegraphics[width=.11\linewidth]{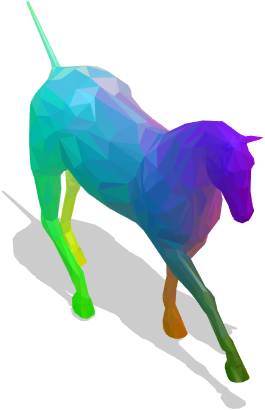}
    \includegraphics[width=.08\linewidth]{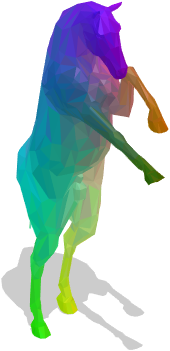}
    \includegraphics[width=.11\linewidth]{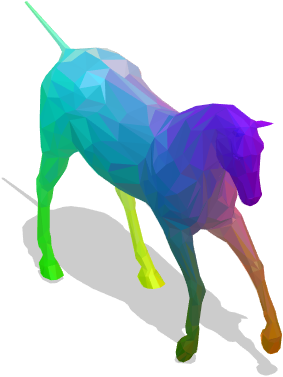}
     \\
    \hline \vspace{-0.35cm} &  \\ 
     & \textsc{HiPPI} (bijective \cmark, cycle-consistent \cmark) \\
    & 
    \includegraphics[width=.11\linewidth]{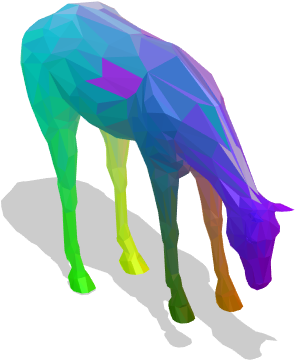}
    \includegraphics[width=.11\linewidth]{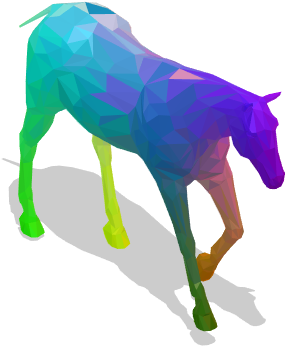}
    \includegraphics[width=.13\linewidth]{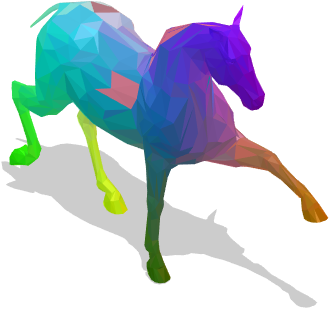}
    \includegraphics[width=.11\linewidth]{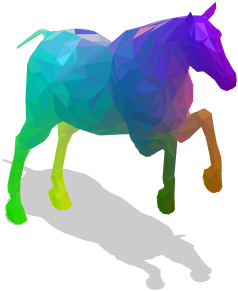}
    \includegraphics[width=.11\linewidth]{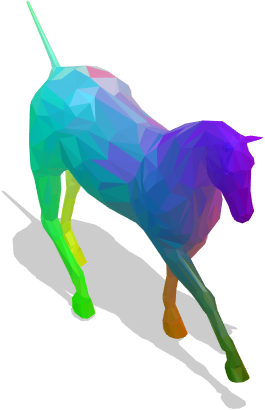}
    \includegraphics[width=.08\linewidth]{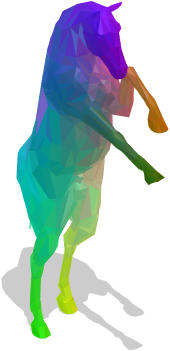}
    \includegraphics[width=.11\linewidth]{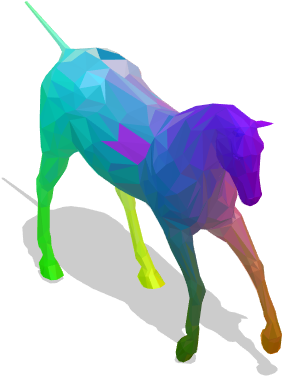}
    \\
    \hline \vspace{-0.35cm} &  \\ 
    & \textsc{ZoomOut+Sync} (bijective \cmark, cycle-consistent \cmark) \\
    & 
    \includegraphics[width=.11\linewidth]{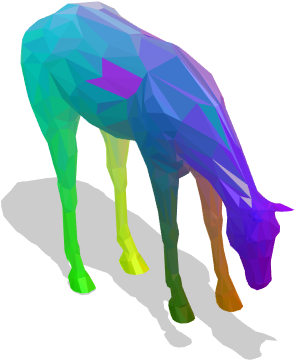}
    \includegraphics[width=.11\linewidth]{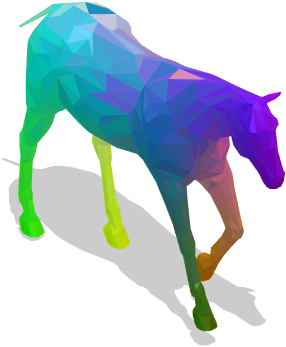}
    \includegraphics[width=.13\linewidth]{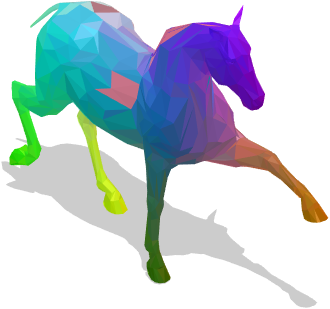}
    \includegraphics[width=.11\linewidth]{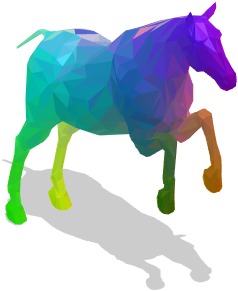}
    \includegraphics[width=.11\linewidth]{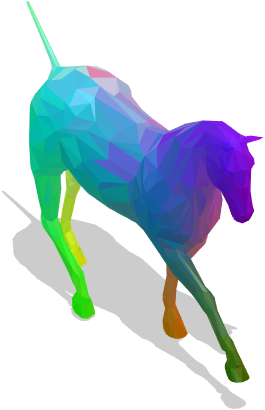}
    \includegraphics[width=.08\linewidth]{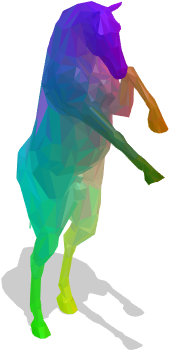}
    \includegraphics[width=.11\linewidth]{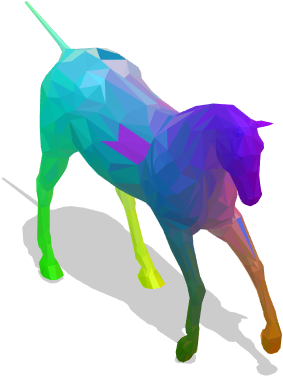}
    \\
    \hline \vspace{-0.35cm} &  \\ 
    & \textsc{ZoomOut} (bijective \xmark, cycle-consistent \xmark) \\
    & 
    \includegraphics[width=.11\linewidth]{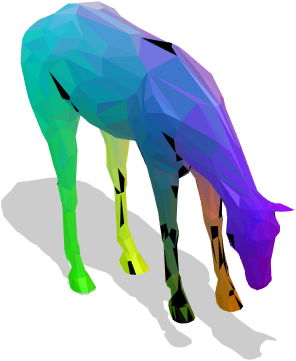}
    \includegraphics[width=.11\linewidth]{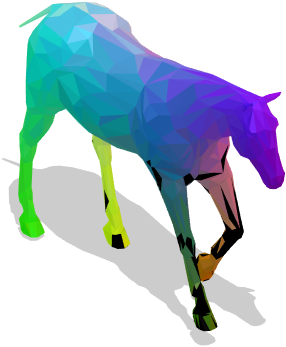}
    \includegraphics[width=.13\linewidth]{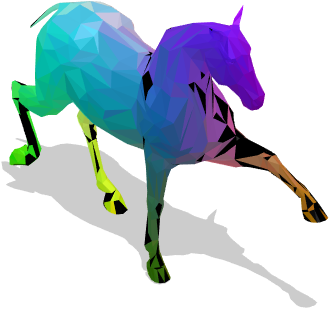}
    \includegraphics[width=.11\linewidth]{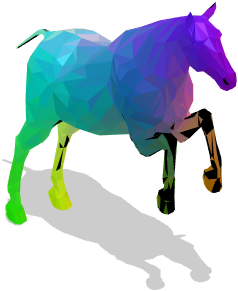}
    \includegraphics[width=.11\linewidth]{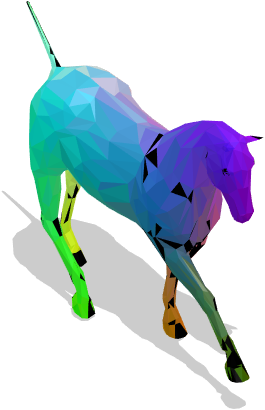}
    \includegraphics[width=.08\linewidth]{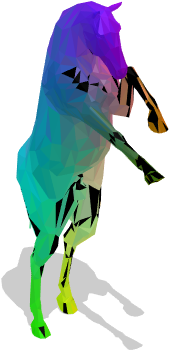}
    \includegraphics[width=.11\linewidth]{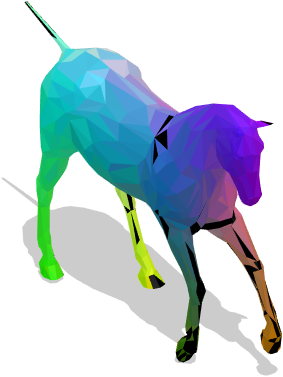}
    \\
    \hline \vspace{-0.35cm} &  \\ 
    & \textsc{ConsistentZoomOut} (bij. \xmark, cycle-cons. \xmark$^\ddagger$) \\
    & 
    \includegraphics[width=.11\linewidth]{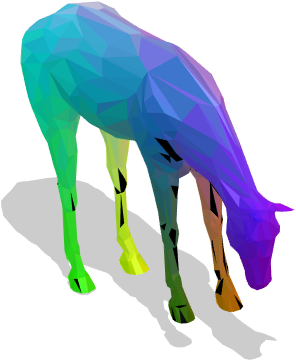}
    \includegraphics[width=.11\linewidth]{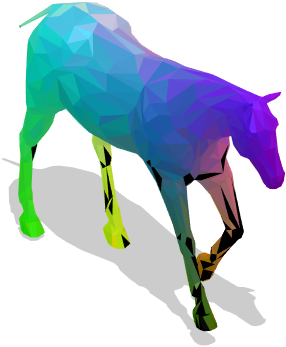}
    \includegraphics[width=.13\linewidth]{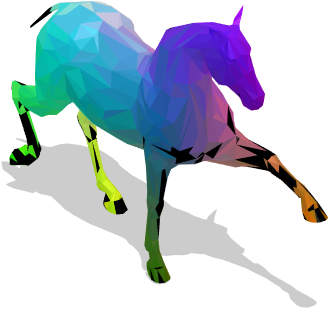}
    \includegraphics[width=.11\linewidth]{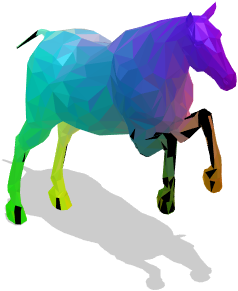}
    \includegraphics[width=.11\linewidth]{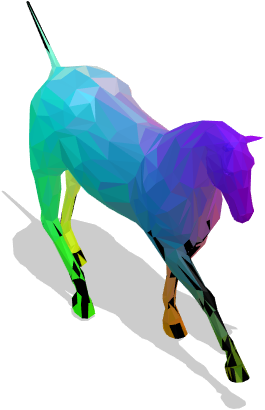}
    \includegraphics[width=.08\linewidth]{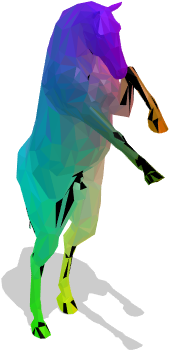}
    \includegraphics[width=.11\linewidth]{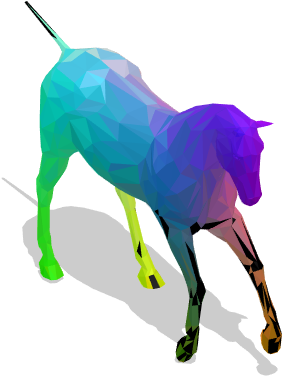}
\end{tabular}
    \caption{Qualitative examples of correspondences on TOSCA horse. Black indicates no matching due to non-bijectivity. Our results contain the best results and are cycle-consistent. $^\ddagger$\textsc{ConsistentZoomOut} obtains cycle-consistent $\fm_{ij}$, but not $P_{ij}$. (Best viewed magnified on screen)}
    \label{fig:toscaHorse} 
\end{figure*}

\begin{figure*}[tbh]
\footnotesize
\centering
\begin{tabular}{c|c}
    Colour legend & Ours (bijective \cmark, cycle-consistent \cmark) \\
    \includegraphics[width=.115\linewidth]{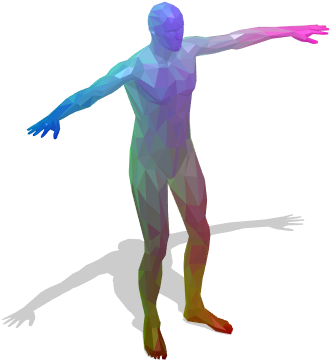} &
    \includegraphics[width=.08\linewidth]{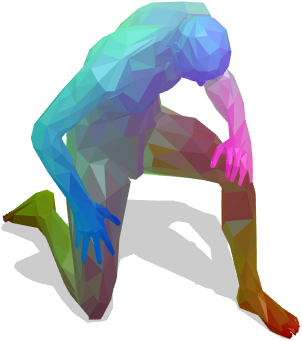}
    \includegraphics[width=.065\linewidth]{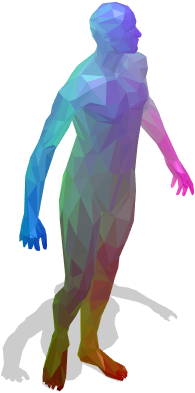}
    \includegraphics[width=.08\linewidth]{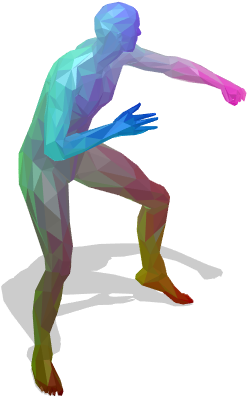}
    \includegraphics[width=.11\linewidth]{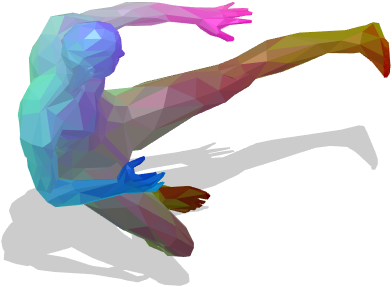}
    \includegraphics[width=.09\linewidth]{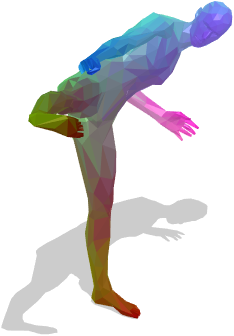}
    \includegraphics[width=.08\linewidth]{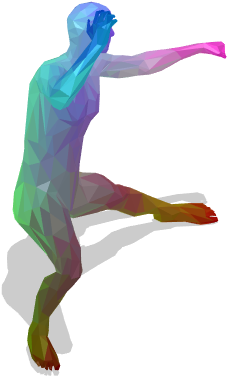}
    \includegraphics[width=.08\linewidth]{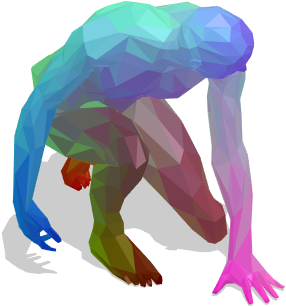}
    \includegraphics[width=.06\linewidth]{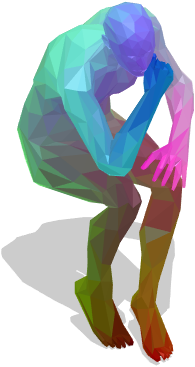}
    \includegraphics[width=.07\linewidth]{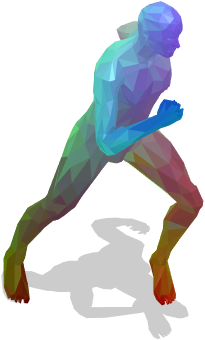}
    \includegraphics[width=.08\linewidth]{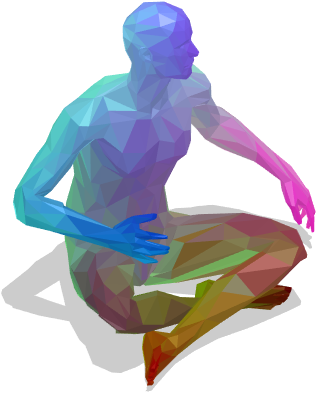}
     \\
    \hline \vspace{-0.35cm} &  \\ 
     & \textsc{HiPPI} (bijective \cmark, cycle-consistent \cmark) \\
    & 
    \includegraphics[width=.08\linewidth]{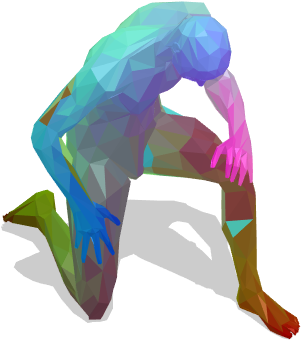}
    \includegraphics[width=.065\linewidth]{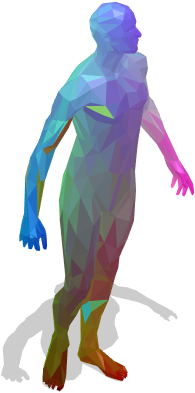}
    \includegraphics[width=.08\linewidth]{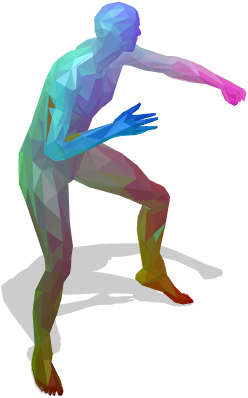}
    \includegraphics[width=.11\linewidth]{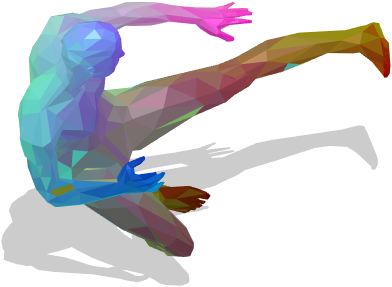}
    \includegraphics[width=.09\linewidth]{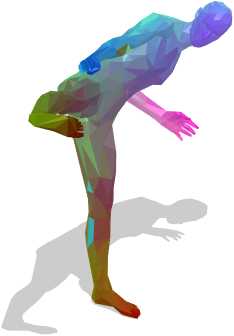}
    \includegraphics[width=.08\linewidth]{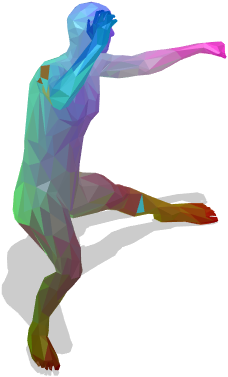}
    \includegraphics[width=.08\linewidth]{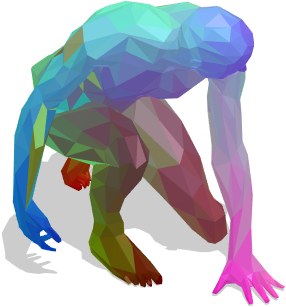}
    \includegraphics[width=.06\linewidth]{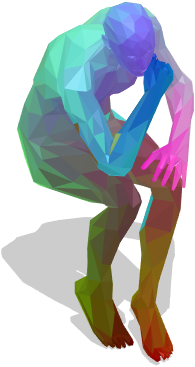}
    \includegraphics[width=.07\linewidth]{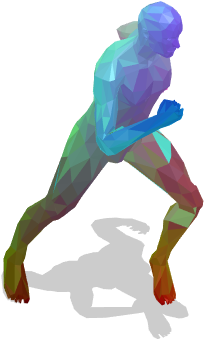}
    \includegraphics[width=.08\linewidth]{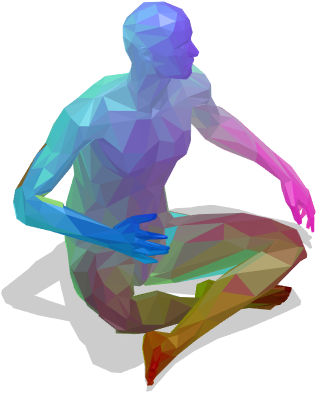}
    \\
    \hline \vspace{-0.35cm} &  \\ 
    & \textsc{ZoomOut+Sync} (bijective \cmark, cycle-consistent \cmark) \\
    & 
    \includegraphics[width=.08\linewidth]{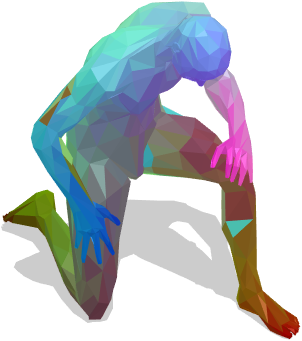}
    \includegraphics[width=.065\linewidth]{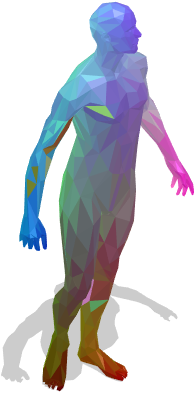}
    \includegraphics[width=.08\linewidth]{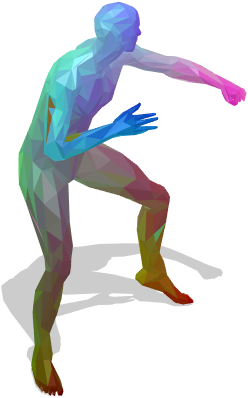}
    \includegraphics[width=.11\linewidth]{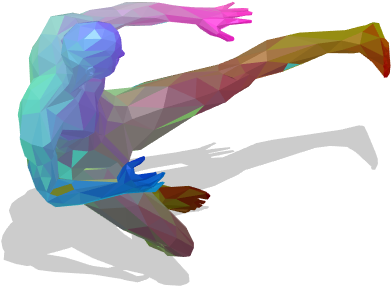}
    \includegraphics[width=.09\linewidth]{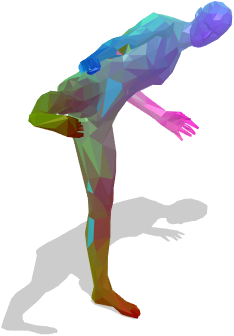}
    \includegraphics[width=.08\linewidth]{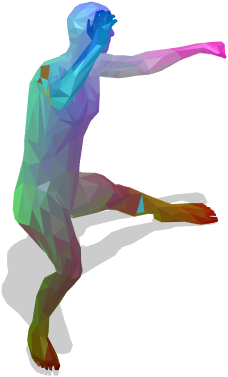}
    \includegraphics[width=.08\linewidth]{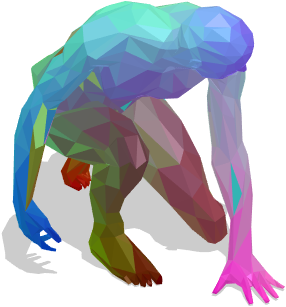}
    \includegraphics[width=.06\linewidth]{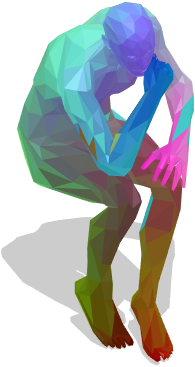}
    \includegraphics[width=.07\linewidth]{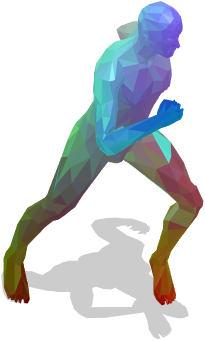}
    \includegraphics[width=.08\linewidth]{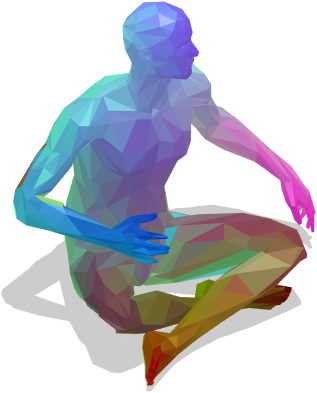}
    \\
    \hline \vspace{-0.35cm} &  \\ 
    & \textsc{ZoomOut} (bijective \xmark, cycle-consistent \xmark) \\
    & 
    \includegraphics[width=.08\linewidth]{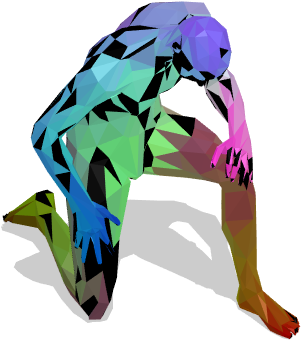}
    \includegraphics[width=.065\linewidth]{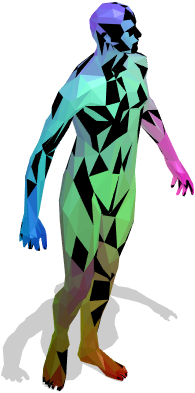}
    \includegraphics[width=.08\linewidth]{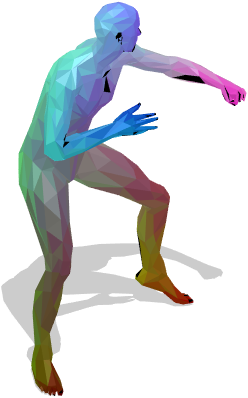}
    \includegraphics[width=.11\linewidth]{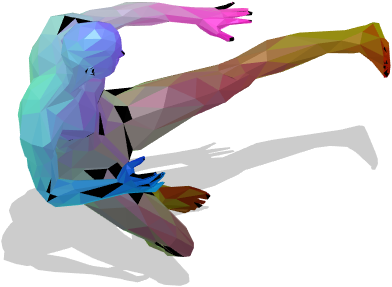}
    \includegraphics[width=.09\linewidth]{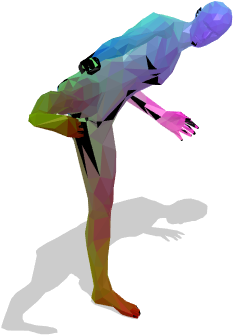}
    \includegraphics[width=.08\linewidth]{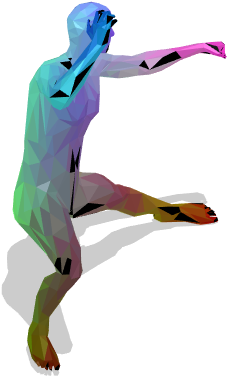}
    \includegraphics[width=.08\linewidth]{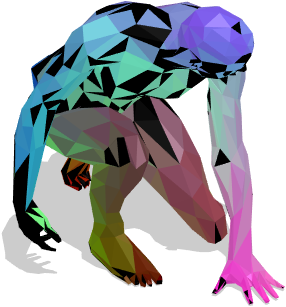}
    \includegraphics[width=.06\linewidth]{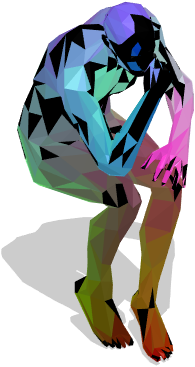}
    \includegraphics[width=.07\linewidth]{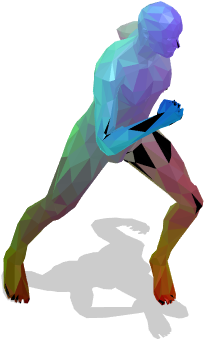}
    \includegraphics[width=.08\linewidth]{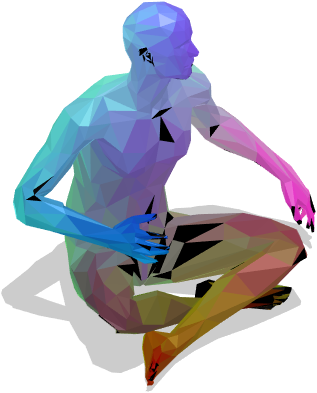}
    \\
    \hline \vspace{-0.35cm} &  \\ 
    & \textsc{ConsistentZoomOut} (bij. \xmark, cycle-cons. \xmark$^\ddagger$) \\
    & 
    \includegraphics[width=.08\linewidth]{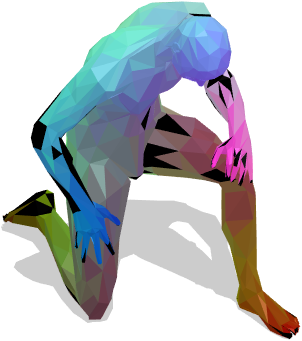}
    \includegraphics[width=.065\linewidth]{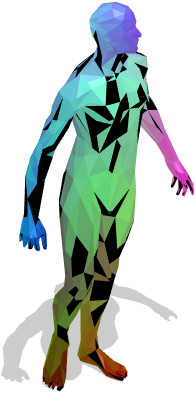}
    \includegraphics[width=.08\linewidth]{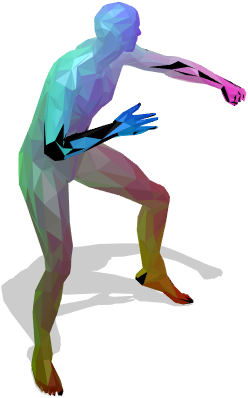}
    \includegraphics[width=.11\linewidth]{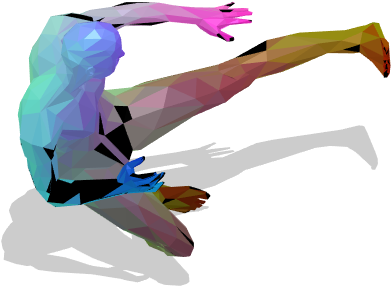}
    \includegraphics[width=.09\linewidth]{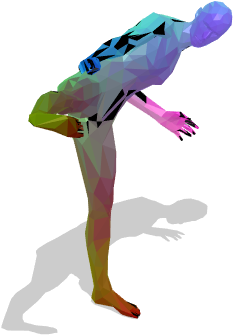}
    \includegraphics[width=.08\linewidth]{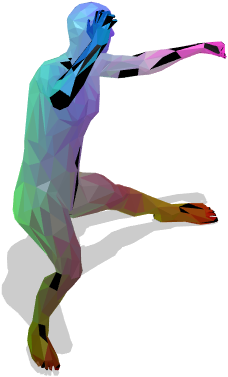}
    \includegraphics[width=.08\linewidth]{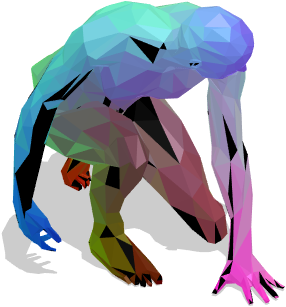}
    \includegraphics[width=.06\linewidth]{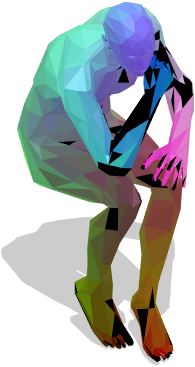}
    \includegraphics[width=.07\linewidth]{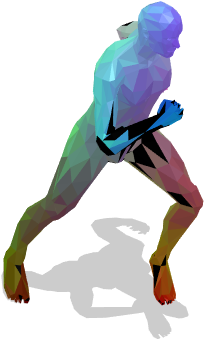}
    \includegraphics[width=.08\linewidth]{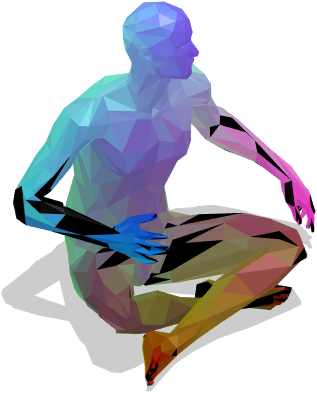}
\end{tabular}
    \caption{Qualitative examples of correspondences on TOSCA michael. Black indicates no matching due to non-bijectivity. $^\ddagger$\textsc{ConsistentZoomOut} obtains cycle-consistent $\fm_{ij}$, but not $P_{ij}$. (Best viewed magnified on screen)}
    \label{fig:toscaMichael} 
\end{figure*}

\end{document}